\documentclass{article}
\usepackage{iclr2027_conference,times}

\usepackage{amsmath,amsfonts,bm}

\def\eqref#1{equation~\ref{#1}}

\def\1{\bm{1}}

\DeclareMathAlphabet{\mathsfit}{\encodingdefault}{\sfdefault}{m}{sl}
\SetMathAlphabet{\mathsfit}{bold}{\encodingdefault}{\sfdefault}{bx}{n}

\usepackage{amsthm}
\usepackage[colorlinks=true,citecolor=hzBuilder,linkcolor=hzInk,urlcolor=hzViolet]{hyperref}
\definecolor{hzBuilder}{HTML}{2A78D6}
\definecolor{hzInk}{HTML}{0B0B0B}
\definecolor{hzViolet}{HTML}{4A3AA7}
\usepackage{url}
\usepackage{graphicx}
\usepackage{algorithm,algorithmic}
\usepackage{booktabs,makecell,xcolor,colortbl}
\graphicspath{{figures/}}
\newcommand{\method}{\textup{\textsc{HMA}}}
\newcommand{\methodfull}{\textit{Humanize MLE Agents} (\textbf{\method})}
\newcommand{\crsmdp}{\textup{\textsc{CR-SMDP}}}
\newcommand{\crsmdpfull}{\textit{context-reset semi-Markov decision process} (\textbf{\crsmdp})}
\newcommand{\nta}{\textup{\textsc{NTA}}}
\newcommand{\ntafull}{\textit{Natural-Termination Alternation} (\textbf{\nta})}

\newcommand{\gain}[1]{\textcolor{green!60!black}{\scalebox{0.75}{\tiny{$+$#1}}}}
\newtheorem{theorem}{Theorem}[section]
\newtheorem{corollary}[theorem]{Corollary}
\newtheorem{proposition}[theorem]{Proposition}
\newtheorem{lemma}[theorem]{Lemma}
\theoremstyle{definition}
\newtheorem{assumption}[theorem]{Assumption}
\newtheorem{definition}[theorem]{Definition}
\theoremstyle{plain}
\newtheorem{appendixtheorem}[theorem]{Theorem}
\theoremstyle{definition}

\usepackage{capt-of}

\iclrfinalcopy
\title{Humanize MLE Agents: Multi-Agent Session Alternation Beyond Single-Agent Limits}

\author{\\
\textbf{Changye Li$^{1}$,
Zijian Zhang$^{1,2}$,
Muheng Li$^{1}$,
Yiduo Pang$^{3}$,
Tingjun Wu$^{1}$} \\
\textbf{Menghan Li$^{1}$,
Yitong Liu$^{1}$,
Hongkai Guo$^{1}$,
Mengchi Chen$^{1}$,
Chenyu He$^{3}$} \\
\textbf{Kaisheng Ma$^{1}$,
Wei Xu$^{1}$,
Song Han$^{2,4}$,
Sihao Liu$^{2}$,
Ligeng Zhu$^{2}$}
\\\\
$^{1}$Tsinghua University,
$^{2}$NVIDIA,
$^{3}$Zhejiang University,
$^{4}$MIT
}

\begin{document}

\maketitle
\lhead{Preprint}

\begin{abstract}
Existing LLM-based agent approaches for long-horizon machine learning engineering (MLE) often use context engineering, implemented through custom harnesses that manage tools, skills, and accumulated experience.
However, we find that as single-session execution lengthens, responses shorten and additional experiments yield diminishing returns, which we term \emph{long-horizon deceleration}.
Moreover, the advantages of custom harnesses vanish or reverse as the underlying model improves, reflecting \emph{harness staleness}.
Together, these observations motivate a workflow-level approach in which agents build on one another's work across sessions.
We model these alternating agent sessions as a \crsmdpfull{} and establish sufficient conditions for alternation to exceed single-agent reference limits, derive an explicit experiment-budget bound, and characterize when session caps preserve this advantage.
Guided by this analysis, we introduce \methodfull{}, which alternates two agents on a shared workspace, limits each session to at most $k$ experiments and resets the context between sessions, without additional task-specific harness or prior knowledge supplements.
Across 75 Kaggle tasks in MLE-bench, \textbf{HMA} reaches a \emph{state-of-the-art} any-medal rate of 78.2\%, \textbf{exceeding the strongest open-source baseline by $8.0$ points with a quarter of its time budget}.
Across five configurations spanning four heterogeneous agent pairs, HMA improves the three-run any-medal rate by $6.9$ points on average over the mean of the constituent baselines and by $3.3$ points on average over the stronger constituent.
Beyond MLE-bench, in a supplementary AOPT kernel-optimization problem, \method{} reaches 878 simulated execution cycles, achieving \emph{a 10th-ranked score} on the Paradigm community leaderboard.
\par\noindent\textbf{Code:} \url{https://github.com/humanfia/hma}

\end{abstract}

\section{Introduction}
\label{sec:hma-introduction}

LLM-based agents tackle machine learning engineering (MLE) through repeated experiments: writing training code, running it, and evaluating candidate solutions~\citep{chan2024mlebench,toledo2025airesearch,karpathy2026autoresearch}.
For long-horizon tasks, producing a working solution is only the beginning: the agent must keep finding ways to improve it.
As further gains become harder to find, the agent is expected to spend more tokens reasoning about the limitations of the current solution and deciding the next action.
In our experiment across five models, however, mean output tokens per response---a proxy for reasoning intensity---decline over time as score gains diminish (Figure~\ref{fig:hma-token-growth}).
We refer to this phenomenon as \emph{long-horizon deceleration} and analyze its implications for prior agent design.

Existing approaches often seek to sustain improvement through context engineering, supplying Markdown instructions, skills, and accumulated knowledge to guide an agent's next steps~\citep{zhu2026mlmaster2,du2026mlevolve,karpathy2026autoresearch,rajasekaran2025context}.
However, guidance that helps one model need not help its successor.
For example, OpenAI recommends rewriting \texttt{AGENTS.md}, skills, and task prompts when upgrading models~\citep{openai2026astra}, while Anthropic advises removing instructions for checks that newer models can already perform on their own~\citep{anthropic2026opus5prompting}.
As model capabilities improve, guidance that once compensated for their limitations can thus become redundant or counterproductive.
A similar pattern emerges when we compare three custom MLE harnesses that manage a model's tools, context, and execution loop~\citep{zhu2026mlmaster2,du2026mlevolve,zhao2026scienceflow}.
Even under longer execution budgets, the relative gains of these specialized harnesses diminish or reverse as the base model improves, as shown in Table~\ref{tab:hma-harness-staleness}.
We refer to this decline in relative performance as \emph{harness staleness}, which indicates the limitations of custom harness design.

These limitations motivate workflow-level coordination across \emph{agent sessions}, each defined as an agent's execution from context initialization to termination.
In MLE, code, trained models, and evaluated candidates persist in a shared workspace, allowing successive agents to build on prior work~\citep{zhao2026scienceflow,young2025harnesses}.
This persistence enables complementary roles: the agent that discovers a useful candidate need not be the one best suited to refine it, and subsequent changes can provide a new starting point when the first agent returns.
We therefore ask: \emph{Can structured alternation convert complementary agent capabilities into a consistent advantage over the best single agent under a matched budget?}

\begin{figure}[t]
 \centering
 \includegraphics[width=\linewidth]{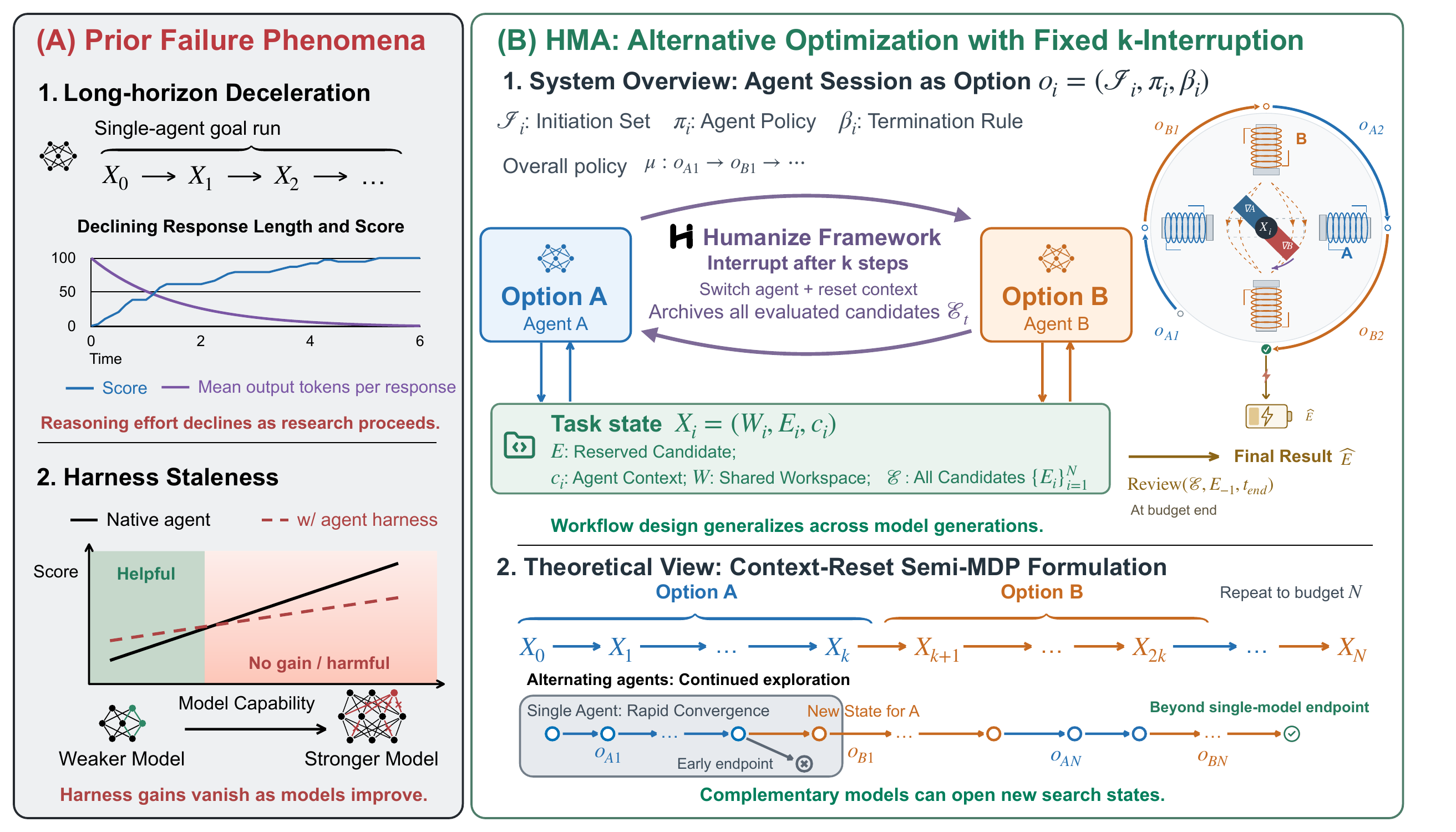}
 \vspace{-8pt}
 \caption{\textbf{\method{} overview.} \textbf{(A)} Declining output tokens per response accompany diminishing score gains, indicating decelerating single-agent exploration; harness gains can vanish or reverse as models improve. \textbf{(B)} \method{} alternates native agents after at most $k$ experiments, resetting context while preserving shared work and candidates for final review. The motor analogy illustrates options $A$ and $B$ alternately energizing shared state $X_i$ toward the final result $\widehat E$. The \crsmdp{} view captures how complementary options can extend exploration beyond single-agent limits.}
 \label{fig:hma-overview}
\end{figure}

To address this question, we represent alternating agent sessions as a \crsmdpfull{}.
We establish sufficient conditions and derive experiment-budget bounds under which alternation exceeds single-agent reference limits, and identify when fixed-$k$ interruption retains this advantage (Sections~\ref{sec:hma-preliminaries}--\ref{sec:hma-fixed-k}).
Building on these results, we introduce \methodfull{}, which coordinates two native agents---complete model--harness pairs using official harnesses---through a shared workspace (Figure~\ref{fig:hma-overview}).
Each session is capped at $k$ experiments; after each handoff, the other agent continues with the shared work and a reset context.
The workflow leaves each agent's tools, context management, and execution loop unchanged~\citep{bolin2026codexloop,anthropicClaudeCodeWorks}, and provides no task-specific knowledge at initialization. In summary, our contributions are:
\begin{itemize}
\item \textbf{We identify single-agent limitations and show how \method{} sustains exploration.}
We characterize \emph{long-horizon deceleration} across five models and \emph{harness staleness} across three custom MLE harnesses in Section~\ref{sec:hma-motivation}. In contrast, from the early to the late phase of execution, \method{} retains 87.3\% of its mean output tokens per response versus 60.8\% for matched native single-agent baselines, while achieving 118.2\% of their post-2h medal gain and 110.7\% of their final medal rate (Figure~\ref{fig:hma-all-growth}).

\item \textbf{Alternation and workflow design extend single-agent limits.}
Our \crsmdp{} analysis establishes sufficient conditions and experiment-budget bounds for exceeding single-agent reference limits (Theorem~\ref{thm:hma-advantage}), and characterizes when fixed-$k$ interruption preserves this advantage (Theorem~\ref{thm:hma-interruption}).
Algorithm~\ref{alg:hma} instantiates these principles with context reset, alternating agents, and final review of candidates.

\item \textbf{\method{} achieves state-of-the-art any-medal performance with cross-pair generality.}
On MLE-bench, \method{} achieves a \emph{state-of-the-art} any-medal rate of 78.2\%, exceeding the strongest published open-source workflow by 8.0 points.
Across five configurations spanning four heterogeneous agent pairs, \method{} improves the any-medal rate by 6.9 points on average over the mean of the constituent baselines and by 3.3 points on average (up to 5.8 points) over the stronger constituent (Section~\ref{sec:hma-main-comparison}).
We also explore applicability beyond MLE-bench in a supplementary 48-hour kernel-optimization case study on Anthropic's Original Performance Take-Home (AOPT), reaching 878 simulated execution cycles (Section~\ref{sec:hma-aopt-cap3}).
\end{itemize}

\section{Motivation: Limits of Single-Agent Exploration}
\label{sec:hma-motivation}

\begin{figure}[h]
\centering
\begin{minipage}[t]{.48\linewidth}
\vspace{0pt}
\centering
\includegraphics[width=.90\linewidth]{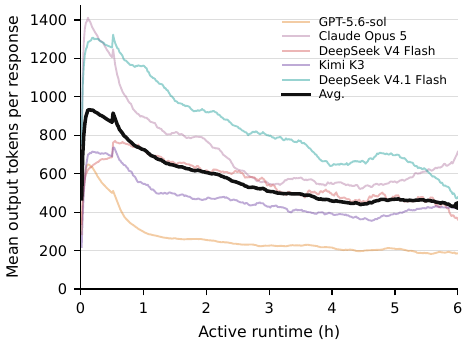}
\vspace{-4pt}
\caption{\textbf{Single-session exploration decelerates over time.}
Each curve reports mean output tokens per recorded model response for one \texttt{/goal} configuration, pooling responses across 75 tasks and repeats. Across configurations, mean response length falls by 29.8\%--45.0\% from hours 0--2 to hours 3--5, indicating slower single-session exploration over time.}
\label{fig:hma-token-growth}
\label{fig:hma-response-token-growth}
\end{minipage}\hfill
\begin{minipage}[t]{.50\linewidth}
\vspace{0pt}
\centering
\makeatletter\def\@captype{table}\makeatother
\caption{\textbf{Specialized harnesses become less effective as the backbone improves.}
We report medal rates on the 16-task MLE-bench subset across DeepSeek-V4-Flash-Preview and DeepSeek-V4.1-Flash. The DeepSeek Harness row is the native baseline. Bold and underlined entries indicate the largest and second-largest rates within each model column. With the weaker backbone, some specialized harnesses match or exceed the native baseline, whereas the stronger backbone consistently favors the native harness. This gap persists despite the longer execution budgets of the specialized workflows, illustrating \emph{harness staleness}.}
\label{tab:hma-harness-staleness}
\footnotesize
\setlength{\tabcolsep}{3pt}
\renewcommand{\arraystretch}{1.12}
\resizebox{\linewidth}{!}{%
\begin{tabular}{@{}l c r r@{}}
\toprule
\textbf{Harness} & \makecell{\textbf{Time}\\\textbf{(h)}}
& \makecell{\textbf{DeepSeek-V4-Flash-Preview}\\\textbf{Medals (\%)}}
& \makecell{\textbf{DeepSeek-V4.1-Flash}\\\textbf{Medals (\%)}} \\
\midrule
DeepSeek Harness & 6 & 16.7 & \textbf{77.1} \\
ScienceFlow & 24 & \textbf{25.0}\,{\scriptsize\textcolor{green!60!black}{+8.3}} & \underline{68.8}\,{\scriptsize\textcolor{red!70!black}{$-8.3$}} \\
ML-Master 2.0 & 24 & 6.3\,{\scriptsize\textcolor{red!70!black}{$-10.4$}} & 37.5\,{\scriptsize\textcolor{red!70!black}{$-39.6$}} \\
MLEvolve & 12 & \underline{18.8}\,{\scriptsize\textcolor{green!60!black}{+2.1}} & 50.0\,{\scriptsize\textcolor{red!70!black}{$-27.1$}} \\
\bottomrule
\end{tabular}%
}
\end{minipage}
\end{figure}

In this section, we investigate two limitations of single-agent MLE workflows: diminishing exploration over extended execution and limited transfer of harness improvements across model versions. These observations motivate \method{}, which coordinates alternating agents from different providers to extend exploration beyond the capabilities of an individual model--harness system.

\paragraph{Long-horizon deceleration.}
We examine five single-session configurations under \texttt{/goal} mode. Figure~\ref{fig:hma-token-growth} reports mean output tokens per response in a trailing 30-minute window, which represent the agent's reasoning intensity. Across configurations, response length decreases by 29.8\%--45.0\% from hours 0--2 to hours 3--5. Further analysis in Figure~\ref{fig:hma-prefix-retention} shows a concurrent rapid saturation of prefix gains, revealing diminishing returns from continued single-agent exploration; detailed measurement definitions and supporting statistics are given in Appendix~\ref{app:goal-token-growth}.

\paragraph{Harness staleness.}
    To counter diminishing single-agent exploration, existing harnesses introduce planning, review, and context engineering~\citep{young2025harnesses,rajasekaran2025context}. With task-specific cold-start knowledge disabled for fairness, our comparison of recent MLE harnesses~\citep{zhu2026mlmaster2,du2026mlevolve,zhao2026scienceflow} in Table~\ref{tab:hma-harness-staleness} reveals \emph{harness staleness}: on the weaker DeepSeek-V4-Flash-Preview backbone, specialized harnesses such as ScienceFlow and MLEvolve exceed the official DeepSeek Harness by 8.3 and 2.1 percentage points, but on the stronger DeepSeek-V4.1-Flash backbone, all three specialized harnesses fall behind the official harness by 8.3--39.6 percentage points despite two to four times its execution budget. This reversal indicates that custom harness scaffolding can become counterproductive as the underlying model improves, whereas the official harness benefits from model--harness alignment acquired during training; a supplementary comparison across four GPT backbones exhibits the same pattern (Appendix~\ref{app:gpt-harness-transfer}). These findings motivate treating each agent as an integrated model--harness system and coordinating agents across providers at the workflow level. Detailed experiment results are presented in Appendix~\ref{app:harness-staleness-results}.

\section{Method}
\label{sec:hma-method}

In this section, we formulate \methodfull{} as a context-reset semi-MDP (Section~\ref{sec:hma-preliminaries}), derive conditions and experiment budgets for alternation to exceed single-agent reference limits (Section~\ref{sec:hma-main-theory}), analyze fixed-$k$ interruption and empirical cap selection (Section~\ref{sec:hma-fixed-k}), and present the complete \method{} workflow (Section~\ref{sec:hma-system}).

\subsection{\method{} Preliminaries}
\label{sec:hma-preliminaries}

\paragraph{Workflow policy and task state.}
A multi-agent workflow that alternates between agent sessions is governed by a fixed \emph{policy over options} $\mu$~\citep{sutton1999between}. Without loss of generality, we consider two agents, $A$ and $B$, working on a shared task. The state $x=(W,E,c^{(u)})$ comprises the shared workspace $W$, retained result $E$, and context $c^{(u)}$ of the active agent $u\in\{A,B\}$. Session $i$ of the workflow, executed by agent $u$, is modeled as an option $o_i^{(u)}=(\mathcal I_i^{(u)},\pi_i^{(u)},\beta_i^{(u)})$, with initiation set $\mathcal I_i^{(u)}:=\{x:\mu(o_i^{(u)}\mid x)>0\}$, internal policy $\pi_i^{(u)}$, and termination rule $\beta_i^{(u)}$. Termination creates an \emph{option boundary}, where $\mu$ selects the next option.

\paragraph{Task initialization and result quality.}
For a general MLE task, $N$ is the total experiment budget of the workflow; each primitive step completes one experiment and updates $(W,E)$ within an option. The initial state $x_0=(W_0,E_0,\emptyset)\sim\nu_0$ is drawn from the task-state distribution $\nu_0$, with workspace $W_0$, empty result $E_0$, and empty context. Let $E_n$ be the result retained after $n$ experiments, with normalized quality $q(x):=q(E)\in[0,1]$. Expected quality and target-success probability are $\operatorname{Perf}_\mu(N):=\mathbb E_{\mu,\nu_0}q(E_N)$ and $p_\mu(N;\tau):=\Pr\nolimits_{\mu,\nu_0}(q(E_N)\ge\tau)$ for $\tau\in[0,1]$.

\paragraph{Workflow trajectory, rewards, and values.}
The trajectory is $(x_0,o_0^{(u_0)},x_1,o_1^{(u_1)},\ldots)$, where session $i$ starts at $x_i$ and selects $o_i^{(u_i)}\sim\mu(\cdot\mid x_i)$ with $u_i\in\{A,B\}$. The option kernel is $P^{(u)}(x,x'):=\Pr\nolimits(x_{i+1}=x'\mid x_i=x,u_i=u)$. We set the discount factor to $\gamma=1$ with step reward $r_{n+1}:=q(E_{n+1})-q(E_n)$, and let $\bar r^{(u)}(x)$ denote the expected option reward from $x$. Values $V^\mu(x)$ and $Q^\mu(x,o_i^{(u)})$ are expected future quality increments, with $Q^\mu$ executing $o_i^{(u)}$ first. The semi-MDP Bellman recursion~\citep{bradtke1994reinforcement,dietterich2000hierarchical} gives
\begin{equation}
 Q^\mu(x,o_i^{(u)})=\bar r^{(u)}(x)+\sum\nolimits_{x'}P^{(u)}(x,x')V^\mu(x'),
 \quad V^\mu(x)=\sum\nolimits_u\mu(o_i^{(u)}\mid x)Q^\mu(x,o_i^{(u)}).
 \label{eq:hma-option-bellman}
\end{equation}

\paragraph{Context-Reset SMDP.}
The \method{} workflow is formulated as a \emph{context-reset SMDP} (\crsmdp{}), whose option boundaries preserve $(W,E)$ and reset context to $\emptyset$. For integer cap $k\ge1$, remaining experiment budget $t$, and natural duration $T_i^{(u)}(x_i)$ measured in experiments, option $i$ completes $d_i^{(u)}=\min\{T_i^{(u)}(x_i),k,t\}$ experiments before natural termination, fixed-$k$ interruption, or budget exhaustion. If $T_i^{(u)}(x_i)>k$ almost surely and $t\ge k$, then $\beta_i^{(u)}(\ell)=\mathbf1_{\{\ell\ge k\}}$ and $d_i=k$, where $\ell$ counts completed experiments within the option. The actual number of completed experiments reflects the agent's progress because the agent can stop when it judges that further experiments are unlikely to help. Further formulation details and natural termination statistics appear in Appendices~\ref{app:formulation} and~\ref{app:option-termination}.

\subsection{Alternation Beyond Single-Agent Limits}
\label{sec:hma-main-theory}

For $u\in\{A,B\}$, let $\mathfrak B^{(u)}$ be a finite set of single-agent workflow variants of $u$ with ordinary or specified stronger harnesses. Define the reference workflow family $\mathfrak B_{A,B}:=\mathfrak B^{(A)}\cup\mathfrak B^{(B)}$. All workflows initialize from $x_0\sim\nu_0$ and satisfy the following retention assumption.
\begin{assumption}[Retention of the best result]
\label{ass:hma-retention}
$q(E_{n+1})\ge q(E_n)$ for all $n$.
\end{assumption}
We adopt the fixed-$k$ termination rule from Section~\ref{sec:hma-preliminaries}. For simplicity, assume the alternating policy $\mu_k=(A,B,u_0,k)$ starts with $u_0\in\operatorname*{arg\,max}_{u\in\{A,B\}}\operatorname{Perf}_{(A,B,u,k)}(N)$, where $\bar A=B$ and $\bar B=A$. With option kernels $P_k^{(u)}$, including context reset, define the cycle kernel and reference ceiling:
\begin{equation}
 M_{\mu_k}=P_k^{(u_0)}P_k^{(\bar u_0)},
 \qquad U_*:=\max_{b\in\mathfrak B_{A,B}}\lim_{N\to\infty}\operatorname{Perf}_b(N).
 \label{eq:hma-joint-reference}
\end{equation}
On each task, let $\mathcal X_{\mathrm{ref}}$ and $\mathcal X_{\mu_k}^{\mathrm{alt}}$ contain states whose shared components $(W,E)$ are reachable by the references and alternation, respectively. For $\tau>U_*$, the target set $\mathcal T:=\{x\in\mathcal X_{\mu_k}^{\mathrm{alt}}\setminus\mathcal X_{\mathrm{ref}}:q(x)\ge\tau\}$ contains high-quality states beyond reference reachability; the next assumption requires persistent cooperative access to them.

\begin{assumption}[Cooperative access and use]
\label{ass:hma-access}
Fix $\mu_k$ and $\tau\in(U_*,1]$. In cycle $m$, agents $u$ and $\bar u$ start at $x_{2m}$ and $x_{2m+1}$, respectively. There exist constants $\alpha,\beta\in(0,1]$ such that, for every $m$ with $\Pr\nolimits_{\mu_k}(q(x_{2m})<\tau)>0$,
{\small
\begin{equation}
 \Pr\nolimits_{\mu_k}(x_{2m+1}\in G\mid q(x_{2m})<\tau)\ge\alpha,
 \quad G:=\{x:q(x)<\tau,\ \sum\nolimits_{x'\in\mathcal T}P_k^{(\bar u)}(x,x')\ge\beta\}.
 \label{eq:hma-expansion-conditions}
\end{equation}
}
Here $\alpha$ bounds $u$'s access to states useful to $\bar u$, and $\beta$ bounds $\bar u$'s probability of reaching $\mathcal T$.
\end{assumption}

\begin{theorem}[Joint alternation exceeds the reference limits]
\label{thm:hma-advantage}
Under Assumptions~\ref{ass:hma-retention} and~\ref{ass:hma-access}, let $a=\alpha\beta$ and $F_0=\Pr\nolimits_{\mu_k}(q(x_0)<\tau)$. For finite $k$ and integers $m\ge0$,
\begin{equation}
 \operatorname{Perf}_{\mu_k}(2km)\ge\tau[1-F_0(1-a)^m],
 \qquad U_{\mu_k}:=\lim_{N\to\infty}\operatorname{Perf}_{\mu_k}(N)\ge\tau>U_*.
 \label{eq:hma-main-guarantee}
\end{equation}
\end{theorem}
\begin{proof}[Proof sketch]
With $F_m=\Pr\nolimits_{\mu_k}(q(x_{2m})<\tau)$, persistent access and quality preservation give $F_{m+1}\le(1-a)F_m$. Iteration and $q\ge0$ imply the bound; monotone convergence gives the limit.
\end{proof}

\begin{corollary}[Experiment budget]
\label{cor:hma-budget}
Under Theorem~\ref{thm:hma-advantage}, with $F_0=1$, fixed $p_0\in(0,1)$, and $0<a\ll1$, a sufficient budget for $p_{\mu_k}(N;\tau)\ge p_0$ has the leading form
\begin{equation}
 N_{\mathrm{suff}}(k,p_0,\tau)
 =\frac{2k}{a}\log\frac{1}{1-p_0}\,\bigl[1+O(a)\bigr].
 \label{eq:hma-target-budget}
\end{equation}
\end{corollary}
Taking $p_0=(U_*+\eta)/\tau$, $0<\eta<\tau-U_*$ guarantees mean quality $U_*+\eta$. The budget form also applies to a single agent with a uniform conditional success probability of at least $a>0$. Appendix~\ref{app:advantage} provides the full proofs, exact budget bounds, startup analysis, and a conditional guarantee for the returned candidate, while Section~\ref{sec:hma-ablation} and Table~\ref{tab:hma-handoff-evidence} report finite-budget handoff statistics supporting cooperative access.

\subsection{Fixed-k Interruption and Empirical Choice}
\label{sec:hma-fixed-k}

In this section, we study when fixed-$k$ interruption maintains the advantage in Theorem~\ref{thm:hma-advantage} and motivate short caps empirically. We use \ntafull{}, $\mu_\infty:=(A,B,u_0,\infty)$, as the reference without fixed-$k$ interruption, with $d_i^{(u)}=\min\{T_i^{(u)}(x_i),t\}$.

Let $x_{i,\ell}$ be the state after $\ell$ experiments within option $i$, with entry $x_{i,0}=x_i$ and final state $x_{i,-1}:=x_{i,d_i^{(u_i)}}$ before context reset. At a boundary, $x_{i+1,0}=f(x_{i,-1})$, where $f(W,E,c):=(W,E,\emptyset)$. We define $A^{\mu_\infty}(x_{i,-1}):=V^{\mu_\infty}(x_{i+1,0})-Q^{\mu_\infty}(x_{i,-1},o_i^{(u_i)})$. Here $Q^{\mu_\infty}$ continues the current option, whereas $V^{\mu_\infty}$ starts the reset option; the interruption theorem~\citep[Theorem~2]{sutton1999between} gives nondegradation when $A^{\mu_\infty}(x_{i,-1})\ge0$.

For the capped policy $\mu_k$, define the interrupted option indices $\mathcal J_{\mu_k}(N):=\{i\ge0:\sum_{j=0}^{i}d_j^{(u_j)}<N,\allowbreak\ d_i^{(u_i)}=k<T_i^{(u_i)}(x_{i,0})\}$.  The expected net gain and \nta{} margin are $\mathcal A_{\mu_k}(N):=\mathbb E_{\mu_k}\sum_{i\in\mathcal J_{\mu_k}(N)}A^{\mu_\infty}(x_{i,-1})$ and $\Delta_{\mu_\infty}(N):=\operatorname{Perf}_{\mu_\infty}(N)-U_*$.

\begin{theorem}[Expected advantage under fixed-k interruption]
\label{thm:hma-interruption}
For the optimal initial agent $u_0$, a sufficient condition to exceed $U_*$ is
\begin{equation}
 \Delta_{\mu_\infty}(N)+\mathcal A_{\mu_k}(N)>0
 \quad\Longrightarrow\quad
 \operatorname{Perf}_{\mu_k}(N)>U_*.
 \label{eq:hma-core-preservation}
\end{equation}
\end{theorem}
\begin{proof}[Proof sketch]
Telescoping the option-level Bellman recursion gives $\operatorname{Perf}_{\mu_k}(N)-\operatorname{Perf}_{\mu_\infty}(N)=\mathcal A_{\mu_k}(N)$. Substitution proves the claim; see Appendix~\ref{app:interruption}.
\end{proof}

\begin{corollary}[Upper bound on the required cap]
\label{cor:hma-sufficient-cap}
If $\Delta_{\mu_\infty}(N)>0$ and interruption losses decay geometrically with $k$, then some finite cap $1\le k\le k_0$ satisfies $\operatorname{Perf}_{\mu_k}(N)>U_*$.
\end{corollary}

\paragraph{Empirically selected cap range: $k\le5$.}
Figure~\ref{fig:hma-prefix-retention} compares five \texttt{/goal} and three \nta{} configurations on $n=75$ MLE-bench tasks~\citep{chan2024mlebench}. The prefix curves show that $k=5$ preserves most observed gains with fewer within-option experiments, motivating candidate caps $k\le5$. Appendix~\ref{app:interruption} proves the interruption bound, and Appendix~\ref{app:empirical} details the empirical prefix analysis.

\begin{figure}[h]
\centering
\begin{minipage}[c]{.46\linewidth}
 \centering
 \includegraphics[width=\linewidth]{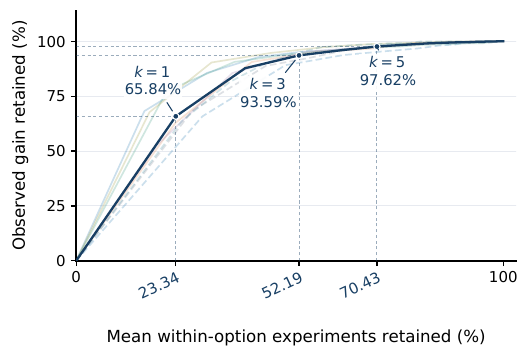}
 \vspace{-12pt}
\caption{\textbf{Short prefixes retain most observed gains.}
The horizontal axis reports the mean percentage of within-option experiments retained by the first $k$ experiments, and the vertical axis reports the corresponding percentage of observed gains retained. Faint curves represent individual configurations, the dark curve averages gain retention equally across configurations, and guides mark $k=1,3,5$. Across the eight configurations, $k=5$ retains 97.62\% of observed gains with a mean within-option experiment retention of 70.43\%, supporting candidate caps $k\le5$.}
 \label{fig:hma-prefix-retention}
\end{minipage}\hfill
\begin{minipage}[c]{.51\linewidth}
\begin{algorithm}[H]
\caption{\method{} with Fixed-$k$ Interruption}
\label{alg:hma}
\footnotesize
\raggedright
\noindent\textbf{Require:} Workflow $\mu_k=(A,B,u_0,k)$; task distribution $\nu_0$; total time $t_{\mathrm{total}}$; final review time $t_{\mathrm{rev}}$.\par
\noindent\textbf{Ensure:} Final candidate $\widehat{E}$.\par
\begin{algorithmic}[1]
\STATE \textcolor{blue}{\textbf{Stage 1: Workflow Initialization}}
\STATE $x_0=(W_0,E_0,\emptyset)\sim\nu_0$.
\STATE $t_{\mathrm{end}}\gets t_{\mathrm{total}}$; $t_{\mathrm{s}}\gets t_{\mathrm{end}}-t_{\mathrm{rev}}$.
\STATE $i\gets0$; $u\gets u_0$; $\mathcal E\gets\emptyset$.
\STATE \textcolor{blue}{\textbf{Stage 2: Alternate Options}}
\WHILE{$t<t_{\mathrm{s}}$}
 \STATE $\ell\gets0$; $x_{i,0}\gets x_i$.
 \WHILE{$\beta_i^{(u)}(\ell,t;k)=0$}
  \STATE $x_{i,\ell+1}\gets\pi_i^{(u)}(x_{i,\ell})$.
  \STATE Append $(E_{i,\ell+1},u)$ to $\mathcal E$.
  \STATE $\ell\gets\ell+1$.
 \ENDWHILE
 \STATE $x_{i+1}\gets f(x_{i,-1})$.
 \STATE $i\gets i+1$; $u\gets\bar u$.
\ENDWHILE
\STATE \textcolor{blue}{\textbf{Stage 3: Final Review}}
\STATE $\widehat{E}\gets E_{-1}$.
\IF{$\lvert\mathcal E\rvert\ge2 \land t<t_{\mathrm{end}}$}
 \STATE $\widehat{E}\gets\textsc{Review}^{(\bar{u}_{-1})}(\mathcal E,E_{-1},t_{\mathrm{end}})$.
\ENDIF
\STATE \textbf{return} $\widehat{E}$.
\end{algorithmic}
\end{algorithm}
\end{minipage}
\end{figure}

\subsection{\method{} Workflow Overview}
\label{sec:hma-system}

In this section, we describe the implementation of the \method{} workflow, including alternating agent execution and final review, as summarized in Algorithm~\ref{alg:hma}.

\paragraph{Agent Execution.}
Algorithm~\ref{alg:hma} implements $\mu_k=(A,B,u_0,k)$ from $x_0\sim\nu_0$. Here $t$ is elapsed time, $t_{\mathrm{end}}=t_{\mathrm{total}}$ is the final deadline, and $t_{\mathrm{s}}=t_{\mathrm{end}}-t_{\mathrm{rev}}$ ends exploration. Agent $u$ executes $\pi_i^{(u)}$ until $\beta_i^{(u)}(\ell,t;k)=\mathbf1\{\ell\ge\min\{T_i^{(u)}(x_i),k\}\ \lor\ t\ge t_{\mathrm{s}}\}=1$. Then $x_{i,-1}=x_{i,\ell}$ and $x_{i+1}=f(x_{i,-1})$ reset context for agent $\bar u$.

\paragraph{Final review.}
The archive $\mathcal E$ stores evaluated candidates and their authors in order; $(E_{-1},u_{-1})$ is the latest pair. We reserve the final $t_{\mathrm{rev}}$ for the other agent $\bar u_{-1}$ to review the archive and select $\widehat E\gets\textsc{Review}^{(\bar u_{-1})}(\mathcal E,E_{-1},t_{\mathrm{end}})$, with $E_{-1}$ as the fallback. This cross-agent check seeks agreement on the best archived candidate without test-set scores, approximating the result-retention objective of Assumption~\ref{ass:hma-retention} at final selection.

\raggedbottom
\section{Experiments}
\label{sec:hma-experiments}

\begin{table}[h]
\centering
\caption{\textbf{\method{} improves overall medal rates over the constituent average.}
Opus $\to$ GPT reaches 78.2\%, 5.8 percentage points above the strongest single-agent baseline. Lite, Medium, and High report any-medal rates by task complexity. Med+ is the share of tasks above the human leaderboard median, and Mean Percentile is the average leaderboard percentile across tasks. Signed gains are in percentage points relative to \textit{Avg. of Single}, the mean of the two \texttt{/goal} baselines. The first listed model starts. Heterogeneous HMA results average three runs, and within-family controls report one run each. Bold and underlined values mark the best and second-best results, respectively.}
\label{tab:hma-main-results}
\begingroup
\small
\setlength{\tabcolsep}{1.2pt}
\renewcommand{\arraystretch}{0.95}
\newcommand{\hmastat}[2]{#1{\tiny$\pm$#2}}
\resizebox{1.0\linewidth}{!}{%
\begin{tabular}{@{}l @{\hskip2.0pt} | c | c c c | c c c | c@{}}
\toprule
\textbf{Workflow}
& \makecell{\textbf{Time}\\\textbf{(h)}}
& \multicolumn{3}{c|}{\textbf{Medal Rate (\%)}}
& \multicolumn{3}{c|}{\textbf{Overall Metrics (\%)}}
& \makecell{\textbf{Mean}\\\textbf{Percentile (\%)}} \\
\cmidrule(lr){3-5}\cmidrule(lr){6-8}\cmidrule(lr){9-9}
& & \makecell{\textbf{Lite}\\\textbf{(22)}}
& \makecell{\textbf{Medium}\\\textbf{(38)}}
& \makecell{\textbf{High}\\\textbf{(15)}}
& \makecell{\textbf{All}\\\textbf{(75)}} & \textbf{Gold} & \textbf{Med+} & \\
\midrule
\multicolumn{9}{l}{\textit{Single-Agent Baselines with Native Harnesses (\texttt{/goal})}} \\
\midrule
Claude Opus 5 & 6 & \hmastat{80.3}{1.5} & \hmastat{78.1}{2.3} & \hmastat{46.7}{0.0} & \hmastat{72.4}{1.6} & \hmastat{50.7}{0.8} & \hmastat{85.3}{1.3} & \hmastat{84.8}{0.6} \\
GPT-5.6-sol & 6 & \hmastat{80.3}{1.5} & \hmastat{69.3}{2.3} & \hmastat{46.7}{6.7} & \hmastat{68.0}{1.3} & \hmastat{47.6}{2.5} & \hmastat{81.3}{0.8} & \hmastat{81.3}{1.3} \\
DeepSeek-V4.1-Flash & 6 & \hmastat{77.3}{2.6} & \hmastat{61.4}{6.3} & \hmastat{37.8}{2.2} & \hmastat{61.3}{4.1} & \hmastat{37.8}{0.9} & \hmastat{74.2}{2.4} & \hmastat{74.7}{2.3} \\
Kimi K3 & 6 & \hmastat{77.3}{2.6} & \hmastat{45.6}{2.3} & \hmastat{42.2}{2.2} & \hmastat{54.2}{1.6} & \hmastat{34.7}{1.3} & \hmastat{73.3}{2.3} & \hmastat{73.8}{1.6} \\
DeepSeek-V4-Flash-Preview & 6 & \hmastat{54.5}{6.9} & \hmastat{24.6}{1.8} & \hmastat{24.4}{5.9} & \hmastat{33.3}{3.5} & \hmastat{15.6}{1.6} & \hmastat{50.2}{3.2} & \hmastat{51.8}{2.8} \\
\midrule
\multicolumn{9}{l}{\textit{Open-Source Agent Baselines}} \\
\midrule
ScienceFlow & 24 & \hmastat{80.3}{1.5} & \hmastat{74.6}{0.9} & \hmastat{44.4}{2.2} & \hmastat{70.2}{1.2} & -- & -- & -- \\
MLEvolve & 12 & \hmastat{80.3}{1.5} & \hmastat{64.0}{0.9} & \hmastat{46.7}{0.0} & \hmastat{65.3}{0.8} & \hmastat{34.7}{0.0} & \hmastat{76.0}{2.3} & -- \\
ML-Master 2.0 & 24 & \hmastat{75.8}{1.5} & \hmastat{50.9}{3.5} & \hmastat{42.2}{2.2} & \hmastat{56.4}{2.5} & \hmastat{19.6}{0.9} & \hmastat{63.1}{1.2} & -- \\
\midrule
\multicolumn{9}{l}{\textit{Within-Family \method{} (Ours)}} \\
\midrule
Claude Opus 5 $\to$ Claude Opus 5 & 6 & \underline{86.4} & \textbf{81.6} & 46.7 & 76.0 & \textbf{54.7} & \underline{86.7} & \underline{86.3} \\
GPT-5.6-sol $\to$ GPT-5.6-sol & 6 & 77.3 & 78.9 & 46.7 & 72.0 & 48.0 & 85.3 & 83.1 \\
GPT-5.6-terra $\to$ GPT-5.6-terra & 6 & 68.2 & 65.8 & 40.0 & 61.3 & 42.7 & 80.0 & 77.4 \\
Kimi K3 $\to$ Kimi K3 & 6 & 72.7 & 52.6 & 40.0 & 56.0 & 28.0 & 73.3 & 71.4 \\
Claude Sonnet 5 $\to$ Claude Sonnet 5 & 6 & 77.3 & 47.4 & 40.0 & 54.7 & 32.0 & 74.7 & 74.1 \\
Claude Opus 5 $\to$ Claude Sonnet 5 & 6 & 81.8 & 73.7 & 46.7 & 70.7 & 48.0 & \underline{86.7} & 86.0 \\
GPT-5.6-sol $\to$ GPT-5.6-terra & 6 & 81.8 & 76.3 & 46.7 & 72.0 & 49.3 & 84.0 & 83.5 \\
\midrule
\multicolumn{9}{l}{\textit{Heterogeneous \method{} (Ours)}} \\
\midrule
\multicolumn{9}{l}{\textbf{Claude Opus 5 $\leftrightarrow$ GPT-5.6-sol}} \\
\quad Avg. of Single & 6 & \hmastat{80.3}{1.1} & \hmastat{73.7}{1.6} & \hmastat{46.7}{3.3} & \hmastat{70.2}{1.0} & \hmastat{49.1}{1.3} & \hmastat{83.3}{0.8} & \hmastat{83.0}{0.7} \\
\rowcolor{gray!15}
\quad \method{} (ours) & 6 & \hmastat{84.8}{1.5}\gain{4.5} & \hmastat{\underline{80.7}}{1.8}\gain{7.0} & \hmastat{\textbf{62.2}}{4.4}\gain{15.6} & \hmastat{\textbf{78.2}}{0.4}\gain{8.0} & \hmastat{53.8}{0.4}\gain{4.7} & \hmastat{\textbf{90.7}}{0.0}\gain{7.3} & \hmastat{\textbf{89.0}}{0.0}\gain{5.9} \\
\midrule
\multicolumn{9}{l}{\textbf{GPT-5.6-sol $\leftrightarrow$ Claude Opus 5}} \\
\quad Avg. of Single & 6 & \hmastat{80.3}{1.1} & \hmastat{73.7}{1.6} & \hmastat{46.7}{3.3} & \hmastat{70.2}{1.0} & \hmastat{49.1}{1.3} & \hmastat{83.3}{0.8} & \hmastat{83.0}{0.7} \\
\rowcolor{gray!8}
\quad \method{} (ours) & 6 & \hmastat{\textbf{87.9}}{1.5}\gain{7.6} & \hmastat{78.1}{0.9}\gain{4.4} & \hmastat{\underline{57.8}}{5.9}\gain{11.1} & \hmastat{\underline{76.9}}{1.8}\gain{6.7} & \hmastat{\underline{54.2}}{1.8}\gain{5.1} & \hmastat{85.8}{0.4}\gain{2.4} & \hmastat{86.2}{0.4}\gain{3.2} \\
\midrule
\multicolumn{9}{l}{\textbf{GPT-5.6-sol $\leftrightarrow$ DeepSeek-V4.1-Flash}} \\
\quad Avg. of Single & 6 & \hmastat{78.8}{1.5} & \hmastat{65.4}{3.4} & \hmastat{42.2}{3.5} & \hmastat{64.7}{2.1} & \hmastat{42.7}{1.3} & \hmastat{77.8}{1.2} & \hmastat{78.0}{1.3} \\
\rowcolor{gray!8}
\quad \method{} (ours) & 6 & \hmastat{78.8}{1.5}\gain{0.0} & \hmastat{77.2}{0.9}\gain{11.8} & \hmastat{51.1}{2.2}\gain{8.9} & \hmastat{72.4}{0.4}\gain{7.8} & \hmastat{51.6}{1.6}\gain{8.9} & \hmastat{84.9}{1.2}\gain{7.1} & \hmastat{84.7}{0.7}\gain{6.7} \\
\midrule
\multicolumn{9}{l}{\textbf{GPT-5.6-sol $\leftrightarrow$ Kimi K3}} \\
\quad Avg. of Single & 6 & \hmastat{78.8}{1.5} & \hmastat{57.5}{1.6} & \hmastat{44.4}{3.5} & \hmastat{61.1}{1.0} & \hmastat{41.1}{1.4} & \hmastat{77.3}{1.2} & \hmastat{77.6}{1.0} \\
\rowcolor{gray!8}
\quad \method{} (ours) & 6 & \hmastat{80.3}{1.5}\gain{1.5} & \hmastat{72.8}{1.8}\gain{15.4} & \hmastat{46.7}{3.8}\gain{2.2} & \hmastat{69.8}{0.9}\gain{8.7} & \hmastat{47.1}{1.8}\gain{6.0} & \hmastat{83.1}{1.2}\gain{5.8} & \hmastat{82.4}{1.0}\gain{4.8} \\
\midrule
\multicolumn{9}{l}{\textbf{DeepSeek-V4.1-Flash $\leftrightarrow$ Kimi K3}} \\
\quad Avg. of Single & 6 & \hmastat{77.3}{1.9} & \hmastat{53.5}{3.4} & \hmastat{40.0}{1.6} & \hmastat{57.8}{2.2} & \hmastat{36.2}{0.8} & \hmastat{73.8}{1.6} & \hmastat{74.3}{1.4} \\
\rowcolor{gray!8}
\quad \method{} (ours) & 6 & \hmastat{77.3}{0.0}\gain{0.0} & \hmastat{59.6}{2.3}\gain{6.1} & \hmastat{42.2}{2.2}\gain{2.2} & \hmastat{61.3}{0.8}\gain{3.6} & \hmastat{40.0}{0.8}\gain{3.8} & \hmastat{77.8}{1.6}\gain{4.0} & \hmastat{77.3}{0.8}\gain{3.0} \\
\bottomrule
\end{tabular}%
}
\endgroup
\end{table}

\subsection{Experimental Setup}
\label{sec:hma-experiment-setup}

\paragraph{Benchmark and evaluation.}
We evaluate on the 75 Kaggle competitions in MLE-bench~\citep{chan2024mlebench}. The primary metric is the percentage of tasks earning any medal. Table~\ref{tab:hma-main-results} also reports medal rates by task complexity, gold medals, performance above the human leaderboard median, and mean leaderboard percentile. We score the returned candidate offline using test-set scores unavailable to the agents.

\paragraph{Execution and repeats.}
We implement \method{} on Humanize~\citep{liu2026humanizejudgementengineeringagentic}, following Algorithm~\ref{alg:hma}. Each MLE-bench run uses one NVIDIA A10 GPU per task and a six-hour budget. For \method{}, each handoff starts a fresh context in the shared workspace, the main comparison caps sessions at $k=5$ accepted experiments, and the final 15 minutes are reserved for review. The five native \texttt{/goal} baselines and five heterogeneous configurations each have three runs per task. Each within-family control has one run per task. Appendix~\ref{app:experiment-details} describes the model--harness mappings, evaluation protocol, and repeats.

\subsection{Main Comparison}
\label{sec:hma-main-comparison}

\paragraph{\method{} outperforms single agents and top published workflows with zero task-specific engineering.}
In Table~\ref{tab:hma-main-results}, Claude Opus 5 $\to$ GPT-5.6-sol reaches a 78.2\% overall any-medal rate, exceeding the strongest single-agent baseline (Claude Opus 5 \texttt{/goal}) by 5.8 points and the top published workflow (ScienceFlow) by 8.0 points, while also leading in Med+ (90.7\%) and Mean Percentile (89.0\%). Notably, \method{} achieves these gains within a six-hour budget, whereas published workflows require 12--24 hours~\citep{zhu2026mlmaster2,du2026mlevolve,zhao2026scienceflow}.

\paragraph{Heterogeneous pairing yields consistent gains over both the constituent mean and the stronger single agent.}
Across all five heterogeneous configurations spanning four model pairs, \method{} raises overall any-medal rates over the mean of the two constituent \texttt{/goal} baselines by 3.6--8.7 percentage points, for an average gain of 6.9 points, while consistently improving Gold, Med+, and Mean Percentile. Relative to the stronger constituent in each pair, which serves as the empirical counterpart of $U_*$, \method{} gains 3.3 points on average: four configurations strictly surpass their stronger constituent by 1.8--5.8 points, and DeepSeek-V4.1-Flash $\to$ Kimi K3 matches it at 61.3\%. Gains remain robust to starting order between comparably capable models: both Opus--GPT directions outperform the 72.4\% Claude Opus 5 \texttt{/goal} baseline, reaching three-run means of 78.2\% with Opus first and 76.9\% with GPT-5.6-sol first, as detailed in Appendix~\ref{app:experiment-details}.

\paragraph{Cross-family pairs achieve the highest medal rates on High-complexity tasks.}
In Table~\ref{tab:hma-main-results}, same-model restarts raise overall any-medal rates over \texttt{/goal} for Claude Opus 5, GPT-5.6-sol, and Kimi K3, while pairing different models within the same family brings no further gain over same-model restarts, with GPT-5.6-sol $\to$ GPT-5.6-terra matching GPT-5.6-sol $\to$ GPT-5.6-sol at 72.0\% and Claude Opus 5 $\to$ Claude Sonnet 5 reaching 70.7\% in Appendix~\ref{app:experiment-details}. However, without cross-family complementarity, both restarts and same-family pairs remain bounded on the hardest tasks, as every \texttt{/goal} baseline and within-family control plateaus at or below 46.7\% on High-complexity tasks. Consistent with Theorem~\ref{thm:hma-advantage}, all five cross-family configurations match or exceed their stronger constituent on High-complexity tasks, improving High-complexity medal rates by 8.0 percentage points on average over the constituent mean and 6.2 points on average over the stronger constituent, and reaching up to 62.2\% with Claude Opus 5 $\to$ GPT-5.6-sol.

\subsection{Experiment Analysis and Ablations}
\label{sec:hma-ablation}

\begin{figure}[!htb]
\centering
\includegraphics[width=.40\linewidth]{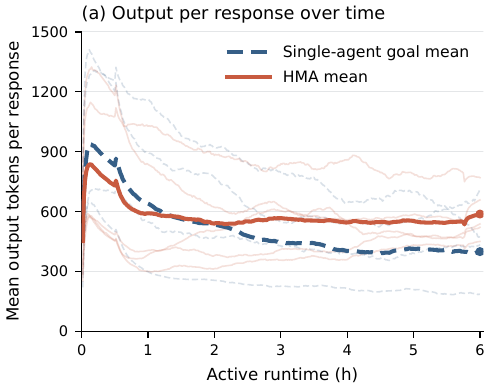}\hspace{.02\linewidth}%
\includegraphics[width=.40\linewidth]{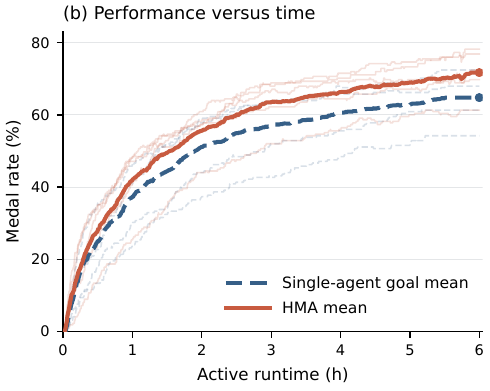}
\caption{\textbf{\method{} sustains exploration over time.}
(a) Mean output tokens per response. (b) Any-medal rate. Bold curves show the five-configuration \method{} mean and the matched constituent \texttt{/goal} mean. Faint curves show individual configurations (Appendix~\ref{app:main-curves}).}
\label{fig:hma-all-growth}
\end{figure}

This section examines how \method{} discovers and retains medal-winning solutions within a six-hour budget. We organize the analysis around four questions: (i) whether alternation mitigates long-horizon deceleration over time, (ii) how often cross-family partners improve inherited results after a handoff, (iii) how the session cap $k$ balances starting-order stability against exploration, and (iv) how reliably final review preserves medal-winning candidates found during exploration.

\paragraph{\method{} mitigates long-horizon deceleration and sustains exploration.}
From hours 0--2 to hours 3--5, \method{} retains 87.3\% of its mean output tokens per response, compared with 60.8\% for the matched \texttt{/goal} baselines (Figure~\ref{fig:hma-all-growth}a). In Figure~\ref{fig:hma-all-growth}(b), \method{} achieves 118.2\% of the matched baseline's post-2h medal gain and reaches 110.7\% of its six-hour any-medal rate across the five heterogeneous configurations, averaging 71.7\% versus 64.8\%, a 6.9-point gain in Table~\ref{tab:hma-main-results}.

\begin{table}[h]
\centering\small
\caption{\textbf{Post-handoff improvements and per-task any-medal comparisons.} Option counts for heterogeneous \method{} average three runs. Tasks above and below compare across-run medal frequencies with the constituent \texttt{/goal} reference across all 75 tasks per configuration. $^\dagger$ uses the sole available constituent baseline.}
\label{tab:hma-handoff-evidence}
\begin{tabular*}{\linewidth}{@{\extracolsep{\fill}}lrrr@{}}
\toprule
\textbf{Configuration} & \textbf{Improved / eligible options (\%)} & \textbf{Tasks above} & \textbf{Tasks below} \\
\midrule
\multicolumn{4}{@{}l}{\textit{Within-Family \method{}}} \\
\midrule
Claude Opus 5 $\to$ Claude Opus 5 & 25/192 (13.0\%) & 9 & 2 \\
GPT-5.6-sol $\to$ GPT-5.6-sol & 66/394 (16.8\%) & 10 & 5 \\
GPT-5.6-terra $\to$ GPT-5.6-terra & 32/235 (13.6\%) & -- & -- \\
Kimi K3 $\to$ Kimi K3 & 26/164 (15.9\%) & 10 & 7 \\
Claude Sonnet 5 $\to$ Claude Sonnet 5 & 44/568 (7.7\%) & -- & -- \\
Claude Opus 5 $\to$ Claude Sonnet 5 & 20/273 (7.3\%) & $5^\dagger$ & $4^\dagger$ \\
GPT-5.6-sol $\to$ GPT-5.6-terra & 47/275 (17.1\%) & $11^\dagger$ & $5^\dagger$ \\
\textbf{Within-family total} & \textbf{260/2,101 (12.4\%)} & \textbf{45} & \textbf{23} \\
\midrule
\multicolumn{4}{@{}l}{\textit{Heterogeneous \method{}}} \\
\midrule
Claude Opus 5 $\to$ GPT-5.6-sol & 38.7/260.7 (14.8\%) & 19 & 2 \\
GPT-5.6-sol $\to$ Claude Opus 5 & 59.0/230.3 (25.6\%) & 17 & 5 \\
GPT-5.6-sol $\to$ DeepSeek-V4.1-Flash & 62.7/320.0 (19.6\%) & 23 & 3 \\
GPT-5.6-sol $\to$ Kimi K3 & 51.3/233.7 (22.0\%) & 19 & 7 \\
DeepSeek-V4.1-Flash $\to$ Kimi K3 & 28.3/235.0 (12.1\%) & 15 & 9 \\
\textbf{Heterogeneous total} & \textbf{240.0/1,279.7 (18.8\%)} & \textbf{93} & \textbf{26} \\
\bottomrule
\end{tabular*}
\end{table}

\paragraph{Cross-family partners improve more inherited results and win across more tasks.}
In Table~\ref{tab:hma-handoff-evidence}, heterogeneous \method{} improves inherited results in 18.8\% of eligible post-handoff sessions, compared with 12.4\% under within-family \method{}, and peaks at 25.6\% for GPT-5.6-sol $\to$ Claude Opus 5; here, a post-handoff session is eligible if an accepted submission exists at handoff, and counts as improved if the incoming agent exceeds the prior best leaderboard percentile. Table~\ref{tab:hma-handoff-evidence} also compares per-task any-medal outcomes against the constituent \texttt{/goal} reference across all 375 configuration--task pairs. Heterogeneous \method{} outperforms the reference on 93 pairs and trails on only 26, remaining net-positive in every configuration, whereas the five evaluated within-family controls record 45 wins against 23 losses.

\paragraph{Additional analyses and ablations.}
Appendix~\ref{app:ablations} presents further analyses and ablations on starting order, session caps, option termination, and final review. In Appendix~\ref{app:starting-order}, capping sessions at $k=5$ prevents the opening agent from monopolizing the six-hour budget, ensuring timely partner participation across all tasks, substantially narrowing starting-order sensitivity relative to natural termination, and outperforming shorter caps $k\in\{1,3\}$. Appendix~\ref{app:option-termination} shows that the fixed cap primarily constrains the initial session while most subsequent sessions terminate naturally. Finally, Appendix~\ref{app:review} shows that cross-agent final review recovers additional medals, raising medal retention from 95.1\% to 97.2\%.

\subsection{Case Study}
\label{sec:hma-case-study}

To illustrate how \method{} compounds gains across handoffs, we examine a GPT--Opus trajectory on INGV Volcanic Eruption Prediction (Figure~\ref{fig:hma-case-ingv}). Across sessions, GPT and Opus directly reuse each other's retrieval features and cached predictions. Consequently, the incoming agent improves the best-so-far MAE in five of six post-handoff sessions. Over 27 experiments, \method{} returns an MAE of 34,234, which is 46.5\% below the best of the six constituent \texttt{/goal} runs. Appendix~\ref{app:case-studies} traces how each agent builds on these shared artifacts and provides additional case studies.

\begin{figure}[h]
\centering
\includegraphics[width=.85\linewidth]{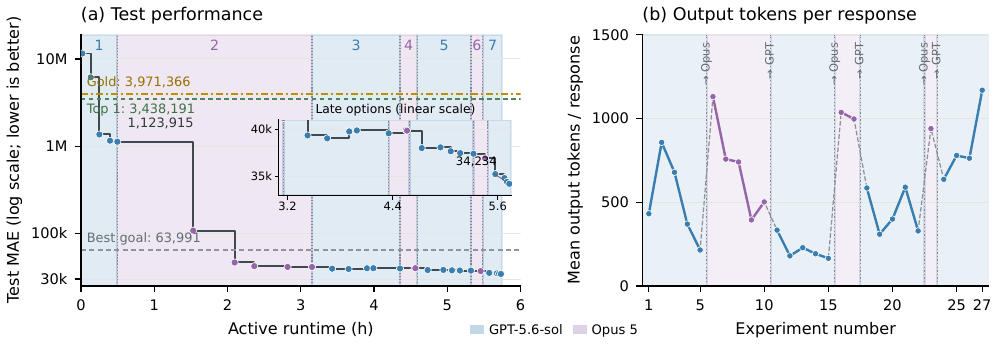}
\vspace{-4pt}
\caption{\textbf{\method{} GPT--Opus alternation sustains improvement on INGV.}
(a) Test MAE over 27 experiments (log scale with linear inset). (b) Mean output tokens per response per experiment. Blue and purple denote GPT and Opus; dotted lines mark handoffs, and horizontal lines mark references.}
\label{fig:hma-case-ingv}
\end{figure}

\subsection{Exploratory Application Beyond MLE-bench}
\label{sec:hma-aopt-cap3}

To test whether \method{} transfers beyond MLE-bench, we conduct an exploratory 48-hour run on Anthropic's Original Performance Take-Home (AOPT), a kernel-optimization task~\citep{anthropic2026performancetakehome}. Claude Opus 5.5 and GPT-6-astra alternate with a cap of three accepted submissions per session. \method{} reaches 878 simulated execution cycles at 34 hours 33 minutes and retains that result in its final candidate (Figure~\ref{fig:hma-aopt-cap3}). For context, 878 cycles matches the reported 10th-ranked score on the Paradigm community leaderboard as of October 8, 2026~\citep{paradigm2026anthropicchallenge}. 

\begin{figure}[h]
\centering
\includegraphics[width=.93\linewidth]{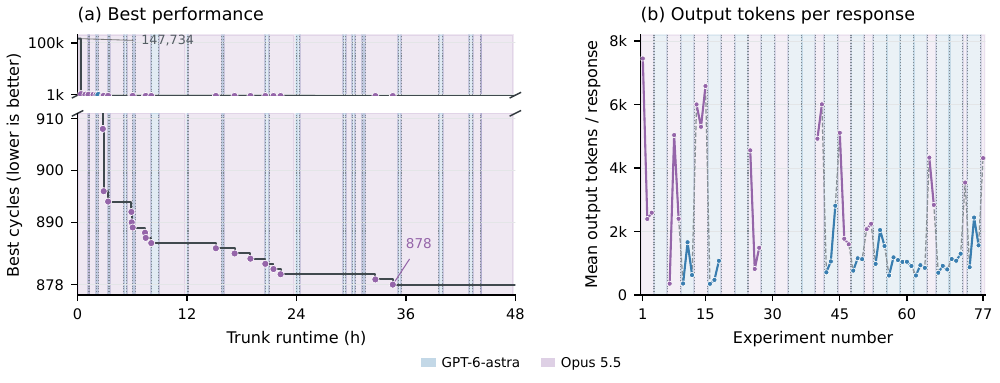}
\vspace{-4pt}
\caption{\textbf{\method{} Opus--GPT alternation on AOPT.}
(a) Best accepted cycle count over 48 hours, using a broken vertical axis (upper: logarithmic; lower: linear) to resolve later improvements. (b) Mean output tokens per response across experiments. Blue and purple regions identify GPT and Opus; dotted lines mark session boundaries.}
\label{fig:hma-aopt-cap3}
\end{figure}

\section{Related Work}
\label{gen_inst}

\paragraph{MLE agents and harness-level systems.}
MLE-bench evaluates autonomous agents on Kaggle competitions~\citep{chan2024mlebench}. Recent MLE systems improve individual agents through search, reflection, memory, branch reuse, recoverable workspaces, or task-conditioned interfaces~\citep{toledo2025airesearch,chen2026mars,du2026mlevolve,zhu2026mlmaster2,zhao2026scienceflow,yu2026recuris,zhang2026jitagent,ye2026cake}. These methods modify how one model acts within its harness. \method{} instead leaves each native model--harness pair unchanged and coordinates complete agents at the workflow level.

\paragraph{Multi-agent and long-horizon workflows.}
Prior workflows sustain long-horizon execution through reusable experience, reflection, tree search, workflow search, communication, role assignment, checkpointing, or independent review~\citep{shinn_reflexion_2023,zhao_expel_2024,zhou2024lats,zhang2025aflow,wu2024autogen,qian2024chatdev,li_agentgit_2025}. Recent handoff studies examine candidate transfer, learned control exchange, and the quality cost of continuing another model's trajectory~\citep{luo2026relay,li2026handoff,ganz2026handoff}. \method{} studies fixed-cap alternation of complete native agents on one shared workspace, with a context reset at every handoff.

\paragraph{Options and semi-MDPs.}
Semi-Markov decision processes model actions with variable duration, while options represent such actions through an internal policy and termination rule~\citep{bradtke1994reinforcement,dietterich2000hierarchical}. The interruption theorem characterizes when terminating an option early preserves or improves continuation value~\citep{sutton1999between}; subsequent work studies learned and off-policy termination~\citep{bacon2017optioncritic,harutyunyan2018offpolicytermination}. \method{} models each agent session as an option and applies this framework to fixed-cap alternation with context reset.

\section{Conclusion}

\method{} alternates native agents, resetting context while preserving shared task state. Our \crsmdp{} analysis gives sufficient conditions and budget bounds for exceeding single-agent reference limits and conditions for preserving this advantage under fixed-cap interruption. Across five MLE-bench configurations under equal six-hour budgets, \method{} improves any-medal rates over both the mean of the constituent baselines and the stronger constituent, reaching 78.2\% with Claude Opus 5 first and GPT-5.6-sol second. A supplementary AOPT case illustrates application to kernel optimization over 48 hours; establishing comparative gains beyond MLE-bench remains future work. Together, these results motivate further study of workflow-level coordination for sustained exploration across iterative optimization tasks.

\subsection*{Acknowledgments}
We thank Zijian Huang, Chengxing Xie, Jing Xiong, Zhengkai Lin, Boshi Zhang, Yan Hu, Dongyun Zou, and Chang Liu for their participation in the early exploration of this project and their contributions to workflow design and experiments.

\subsection*{AI use statement}

In this work, we used generative AI tools for language polishing, grammar checking, and generating portions of the code. The research ideas, methodology, experimental design, and interpretation of results were developed by the authors. We reviewed all AI-assisted text to ensure that it accurately conveyed our intended meaning. LLM-generated code was verified and tested for correctness by two authors. We take full responsibility for the final content of this work, including all text, claims, and artifacts produced with the assistance of generative AI.

\begingroup
\fontencoding{T1}\selectfont
\bibliography{iclr2027_conference}
\endgroup
\bibliographystyle{iclr2027_conference}

\clearpage
\appendix
\setlength{\abovedisplayskip}{10pt plus 2pt minus 2pt}
\setlength{\belowdisplayskip}{10pt plus 2pt minus 2pt}
\setlength{\abovedisplayshortskip}{8pt plus 2pt minus 2pt}
\setlength{\belowdisplayshortskip}{8pt plus 2pt minus 2pt}
\setlength{\jot}{4pt}

\makeatletter
\let\hmaAppendixAddcontentsline\addcontentsline
\renewcommand{\addcontentsline}[3]{%
  \hmaAppendixAddcontentsline{#1}{#2}{#3}%
  \def\hmaAppendixContentsFile{#1}%
  \def\hmaMainContentsFile{toc}%
  \ifx\hmaAppendixContentsFile\hmaMainContentsFile
    \hmaAppendixAddcontentsline{apc}{#2}{#3}%
  \fi
}
\pdfbookmark[0]{Appendix Contents}{hma-appendix-contents}
\section*{Appendix Contents}
\begingroup
\small
\setcounter{tocdepth}{2}
\@starttoc{apc}
\endgroup
\makeatother

\clearpage
\section{Theoretical Details and Proofs}
\label{app:theory}

We organize the formal supplement in the same order as the Method section: Appendix~\ref{app:formulation} defines the finite-budget \crsmdp{} and exact performance curves; Appendix~\ref{app:advantage} proves the cooperative alternation guarantee and budget bound; Appendix~\ref{app:interruption} handles \nta{}, fixed-$k$ interruption, and sufficient-cap conditions; Appendix~\ref{app:empirical} defines the prefix-retention statistic and reports the empirical prefix curves.

\subsection{Model Details and Exact Performance}
\label{app:formulation}

We supplement Section~\ref{sec:hma-preliminaries} with finite-budget definitions, the general option recursion, and exact performance expressions used in the alternation proof.

\begin{definition}[State and finite budget]
\leavevmode\phantomsection\label{def:app-state}
The state from Section~\ref{sec:hma-preliminaries} implicitly includes agent identity, option age, and any scheduling history needed for the Markov property~\citep[Section~3.2]{dietterich2000hierarchical}. For session $o_i^{(u_i)}$ at $x_i$, let $n_i$ be the completed experiment count and $t_i$ the remaining budget. Then
\begin{equation}
 n_i:=\sum_{h=0}^{i-1}d_h,\qquad t_i:=N-n_i,\qquad
 d_i=\min\{T_i^{(u_i)}(x_i),k,t_i\}.
 \label{eq:app-option-budget}
\end{equation}
Here $n_0=0$, and $T_i^{(u)}(x_i)$ is the positive, almost-surely finite natural duration, which may depend on the entry state. Thus early natural termination and a final partial option are included; $k=\infty$ gives \nta{}. Values track the remaining budget, vanish when it is exhausted, and no option starts at $t_i=0$.
\end{definition}

\begin{definition}[Option rewards and transitions]
\leavevmode\phantomsection\label{def:app-options}
For $\gamma\in(0,1]$ and the experiment rewards in Section~\ref{sec:hma-preliminaries}, define
\begin{equation}
 \bar r^{(u)}(x):=\mathbb E\!\left[
 \sum_{\ell=0}^{d_i-1}\gamma^\ell r_{n_i+\ell+1}
 \;\middle|\;x_i=x,\ u_i=u\right].
 \label{eq:hma-option-reward}
\end{equation}
Let $\Phi_{k,t}^{(u)}(x',d\mid x)$ be the joint law of duration and next option-entry state, including context reset, and $P_{k,t}^{(u)}(x,x'):=\sum_d\Phi_{k,t}^{(u)}(x',d\mid x)$. The complete-option kernel is $P_k^{(u)}:=P_{k,\infty}^{(u)}$; finite $t$ truncates the same option. Suppressing fixed $k,t$, the general semi-MDP recursion~\citep{sutton1999between,bradtke1994reinforcement} is
\begin{equation}
 Q^\mu(x,o_i^{(u)})=\bar r^{(u)}(x)
 +\sum_{d,x'}\gamma^d\Phi^{(u)}(x',d\mid x)V^\mu(x').
 \label{eq:app-option-bellman}
\end{equation}
At $\gamma=1$, this reduces to~\eqref{eq:hma-option-bellman}, and total reward is $q(E_N)-q(E_0)$.
\end{definition}

\begin{proposition}[Exact performance expressions]
\leavevmode\phantomsection\label{prop:app-performance}
Let $\mathcal P_\mu$ be the one-experiment kernel under $\mu$ on the augmented state, including option boundaries, and $S_\tau:=\{x:q(x)\ge\tau\}$. After workflow termination, retain the terminal state. Then
\begin{equation}
 p_\mu(N;\tau)=\nu_0\mathcal P_\mu^N\mathbf1_{S_\tau},
 \qquad \operatorname{Perf}_\mu(N)=\nu_0\mathcal P_\mu^Nq,
 \label{eq:app-exact-performance}
\end{equation}
where $\mathbf1_{S_\tau}$ indicates target attainment. These expressions give the probability of reaching $S_\tau$ and the expected final quality, respectively, with $\operatorname{Perf}_\mu(N)\ge\tau p_\mu(N;\tau)$.
\end{proposition}

\begin{proof}[Proof of Proposition~\ref{prop:app-performance}]\itshape
The endpoint law is $\nu_0\mathcal P_\mu^N$; applying it to $\mathbf1_{S_\tau}$ and $q$ gives~\eqref{eq:app-exact-performance}. Taking expectations of $q\ge\tau\mathbf1_{S_\tau}$ yields the inequality.
\end{proof}

\subsection{Alternation Guarantee and Experiment Budget}
\label{app:advantage}

We follow Section~\ref{sec:hma-main-theory}: cooperative access yields a failure recursion, which proves the alternation guarantee and the experiment budget; we then discuss startup effects.

\begin{definition}[Reference ceiling and cycle budget]
\leavevmode\phantomsection\label{def:app-reference}
For the reference family in Section~\ref{sec:hma-main-theory}, bounded quality and Assumption~\ref{ass:hma-retention} give
\begin{equation}
 U_*:=\max_{b\in\mathfrak B_{A,B}}U_b,\qquad
 U_b:=\lim_{N\to\infty}\operatorname{Perf}_b(N)
 =\mathbb E_{b,\nu_0}\!\left[\lim_{N\to\infty}q(E_N)\right].
 \label{eq:app-reference-limit}
\end{equation}
Fix $\mu_k$ and $\tau>U_*$. Write $u$ for the first agent in each cycle: cycle $m$ executes $o_{2m}^{(u)}$ followed by $o_{2m+1}^{(\bar u)}$, with kernel $M_{\mu_k}$. Before total-budget truncation, the experiments used by $m$ cycles are $C_{\mu_k,m}:=\sum_{i=0}^{2m-1}\min\{T_i^{(u_i)}(x_i),k\}$, so $C_{\mu_k,m}\le2km$ for finite $k$. Write $F_m:=\Pr\nolimits_{\mu_k}(q(x_{2m})<\tau)$, where $\Pr\nolimits_{\mu_k}$ and $\mathbb E_{\mu_k}$ include the initial draw from $\nu_0$.
\end{definition}

\begin{lemma}[Cooperative success and failure recursion]
\leavevmode\phantomsection\label{lem:app-failure}
Under Assumptions~\ref{ass:hma-retention} and~\ref{ass:hma-access}, each cycle has conditional success probability at least $a:=\alpha\beta$ whenever $F_m>0$, and $F_{m+1}\le(1-a)F_m$ for all $m\ge0$.
\end{lemma}

\begin{proof}[Proof of Lemma~\ref{lem:app-failure}]\itshape
The $u$-option enters $G$ with conditional probability at least $\alpha$, and the $\bar u$-option reaches $\mathcal T\subseteq S_\tau$ from each $x\in G$ with probability at least $\beta$. By the Markov property, for $F_m>0$,
\begin{equation}
 \Pr\nolimits_{\mu_k}(x_{2m+2}\in S_\tau\mid x_{2m}\notin S_\tau)
 \ge\beta\Pr\nolimits_{\mu_k}(x_{2m+1}\in G\mid x_{2m}\notin S_\tau)
 \ge\alpha\beta=a.
 \label{eq:app-cycle-success}
\end{equation}
Quality retention makes $S_\tau$ absorbing, so $F_{m+1}\le F_m(1-a)$. If $F_m=0$, it also gives $F_{m+1}=0$.
\end{proof}

\begin{appendixtheorem}[Joint alternation exceeds the reference limits; Theorem~\ref{thm:hma-advantage} restated]
\leavevmode\phantomsection\label{thm:app-advantage}
Under Assumptions~\ref{ass:hma-retention} and~\ref{ass:hma-access}, take $a=\alpha\beta$ as a lower bound on conditional success per cycle, and $F_0=\Pr\nolimits_{\mu_k}(q(x_0)<\tau)$. For finite $k$,
\begin{equation}
 \operatorname{Perf}_{\mu_k}(2km)\ge\tau[1-F_0(1-a)^m],
 \qquad U_{\mu_k}:=\lim_{N\to\infty}\operatorname{Perf}_{\mu_k}(N)\ge\tau>U_*.
 \label{eq:app-alternation-guarantee}
\end{equation}
\end{appendixtheorem}

\begin{proof}[Proof of Theorem~\ref{thm:app-advantage}]\itshape
Iterating Lemma~\ref{lem:app-failure} gives $F_m\le F_0(1-a)^m$. Since $C_{\mu_k,m}\le2km$, quality retention and Proposition~\ref{prop:app-performance} yield
\begin{equation}
 \operatorname{Perf}_{\mu_k}(2km)\ge\tau p_{\mu_k}(2km;\tau)
 \ge\tau(1-F_m)\ge\tau[1-F_0(1-a)^m].
 \label{eq:app-finite-cycle-bound}
\end{equation}
Bounded monotone convergence ensures that $U_{\mu_k}$ exists; letting $m\to\infty$ gives $U_{\mu_k}\ge\tau>U_*$. The same proof permits early natural termination when Assumption~\ref{ass:hma-access} holds for the resulting kernels, since $C_{\mu_k,m}\le2km$ still holds.
\end{proof}

\begin{corollary}[Experiment budget; Corollary~\ref{cor:hma-budget} restated]
\leavevmode\phantomsection\label{cor:app-single-stage}
Under the theorem's conditions, let $F_0=1$ and fix $p_0\in(0,1)$. For $0<a\ll1$, a sufficient budget for $p_{\mu_k}(N;\tau)\ge p_0$ has the leading form
\begin{equation}
 N_{\mathrm{suff}}(k,p_0,\tau)
 =\frac{2k}{a}\log\frac{1}{1-p_0}\,\bigl[1+O(a)\bigr].
 \label{eq:app-budget-leading}
\end{equation}
\end{corollary}

\begin{proof}[Proof of Corollary~\ref{cor:app-single-stage}]\itshape
For $0<a<1$, the bound $p_{\mu_k}(2km;\tau)\ge1-(1-a)^m$ reaches $p_0$ once $m\ge\log(1/(1-p_0))/[-\log(1-a)]$. Hence the exact rounded sufficient budget is
\begin{equation}
 N_{\mathrm{suff}}(k,p_0,\tau)
 =2k\left\lceil\frac{\log(1/(1-p_0))}{-\log(1-a)}\right\rceil.
 \label{eq:app-exact-budget}
\end{equation}
Using $-\log(1-a)=a+O(a^2)$ and a rounding error below $2k$ gives $(2k/a)\log(1/(1-p_0))+O(k)$, equivalent to~\eqref{eq:app-budget-leading} for fixed $p_0$. Finally, $p_0=(U_*+\eta)/\tau$, $0<\eta<\tau-U_*$, gives $\operatorname{Perf}_{\mu_k}(N)\ge U_*+\eta$ by Proposition~\ref{prop:app-performance}.
\end{proof}

\paragraph{Guarantee for the returned candidate.}
\label{app:returned-guarantee}
Let $B_N$ be the maximum quality in the archive after $N$ experiments and let $R_N\in[0,1]$ be the quality of the returned candidate. Suppose the archive process satisfies Assumption~\ref{ass:hma-access} with the same quality measure used for $U_*$, and that $\Pr(R_N\ge\tau\mid B_N\ge\tau)\ge\rho_\tau$ whenever the conditioning event has positive probability, uniformly over the budgets considered. Then the archive is monotone and the cycle bound above gives
\begin{equation}
\Pr(R_{2km}\ge\tau)
 \ge \rho_\tau\Pr(B_{2km}\ge\tau)
 \ge \rho_\tau[1-F_0(1-\alpha\beta)^m],
 \qquad F_0=\Pr(B_0<\tau).
\end{equation}
Since $R_N\ge0$, this implies $\liminf_{m\to\infty}\mathbb E[R_{2km}]\ge\rho_\tau\tau$. Thus $\rho_\tau\tau>U_*$ suffices for the returned result to exceed the reference ceiling in long-run expected quality along these budgets. Appendix~\ref{app:review} reports empirical retention at the six-hour any-medal threshold.

\paragraph{Startup and initial-agent effects.}
Startup denotes the early cycles before cooperative progress stabilizes. Because the cycle kernels $P_k^{(A)}P_k^{(B)}$ and $P_k^{(B)}P_k^{(A)}$ can induce different early task states, the initial agent $u_0$ can affect access to $G$, finite-budget quality, and both $\Delta_{\mu_\infty}(N)$ and $\mathcal A_{\mu_k}(N)$ in Section~\ref{sec:hma-fixed-k}. If the two starting orders later achieve similar conditional success rates, their subsequent budget scales $2k/a$ are similar while their early progress can still differ.

\subsection{Interruption Guarantee and Sufficient Caps}
\label{app:interruption}

We supplement Section~\ref{sec:hma-fixed-k} by extending the alternation guarantee to \nta{} (Corollary~\ref{cor:app-nta}), defining interruption advantage (Definition~\ref{def:app-interruption}), proving the fixed-$k$ interruption guarantee (Theorem~\ref{thm:app-interruption}), and deriving explicit sufficient-cap bounds (Corollaries~\ref{cor:app-candidate-cap}--\ref{cor:app-explicit-cap}).

\begin{corollary}[Natural termination and quality limits]
\leavevmode\phantomsection\label{cor:app-nta}
For \ntafull{}, use $\mu_\infty=(A,B,u_0,\infty)$. Under Assumptions~\ref{ass:hma-retention} and~\ref{ass:hma-access} and almost-surely finite option durations, let $M_\tau:=\inf\{m\ge0:q(x_{2m})\ge\tau\}$. Then $C_{\mu_\infty,M_\tau}<\infty$ almost surely and $U_{\mu_\infty}\ge\tau>U_*$. If \nta{} and a finite cap satisfy access for every $\tau\in(U_*,U^\dagger)$, where $U^\dagger\in(U_*,1]$ bounds both reachable quality spaces, then $U_{\mu_k}=U_{\mu_\infty}=U^\dagger$.
\end{corollary}

\begin{proof}[Proof of Corollary~\ref{cor:app-nta}]\itshape
Applying~\eqref{eq:app-cycle-success} to \nta{} gives $\Pr\nolimits_{\mu_\infty}(M_\tau>m)\le F_0(1-a)^m\to0$, with $a=\alpha\beta$. Thus $M_\tau$ and the sum of its option durations are almost surely finite. Monotone quality yields $U_{\mu_\infty}\ge\tau$. Combining this argument with Theorem~\ref{thm:app-advantage} for every $\tau\in(U_*,U^\dagger)$ and using the common upper bound gives equality for both limits.
\end{proof}

\begin{definition}[Interruption advantage and reference margin]
\leavevmode\phantomsection\label{def:app-interruption}
Fix $\mu_k=(A,B,u_0,k)$ and $N$. The workflows $\mu_k$ and $\mu_\infty=(A,B,u_0,\infty)$ share internal policies and boundary updates and differ only in option termination. Following Section~\ref{sec:hma-fixed-k}, $x_{i,0}=x_i$ is option $i$'s entry, $x_{i,-1}=x_{i,d_i^{(u_i)}}$ its endpoint before reset, and $x_{i+1,0}=f(x_{i,-1})$ the next entry, with $q(x_{i+1,0})=q(x_{i,-1})$.

At an added cap interruption, \nta{} would continue from $x_{i,-1}$ with the current context. Ending the option instead gives advantage $A^{\mu_\infty}(x_{i,-1}):=V^{\mu_\infty}(x_{i+1,0})-Q^{\mu_\infty}(x_{i,-1},o_i^{(u_i)})$. Both values use \nta{} with the same remaining experiment budget; the next entry activates the other agent with empty context. The interruption contribution is zero at natural termination or budget exhaustion.

The added interruptions are indexed by
\[
 \mathcal J_{\mu_k}(N):=\left\{i\ge0:\sum_{j=0}^{i}d_j^{(u_j)}<N,\ d_i^{(u_i)}=k<T_i^{(u_i)}(x_{i,0})\right\}.
\]
Thus natural termination takes precedence when it coincides with the cap. The expected net gain is $\mathcal A_{\mu_k}(N):=\mathbb E_{\mu_k}\sum_{i\in\mathcal J_{\mu_k}(N)}A^{\mu_\infty}(x_{i,-1})$, and the \nta{} margin is $\Delta_{\mu_\infty}(N):=\operatorname{Perf}_{\mu_\infty}(N)-U_*$.

\end{definition}

\begin{appendixtheorem}[Expected advantage under fixed-k interruption; Theorem~\ref{thm:hma-interruption} restated]
\leavevmode\phantomsection\label{thm:app-interruption}
For the optimal initial agent $u_0$, a sufficient condition to exceed $U_*$ is
\begin{equation}
 \Delta_{\mu_\infty}(N)+\mathcal A_{\mu_k}(N)>0
 \quad\Longrightarrow\quad
 \operatorname{Perf}_{\mu_k}(N)>U_*.
 \label{eq:app-interruption-condition}
\end{equation}
\end{appendixtheorem}

\begin{proof}[Proof of Theorem~\ref{thm:app-interruption}]\itshape
Let $M$ be the number of options executed under $\mu_k$ within budget $N$, with $x_{M,0}$ terminal, and let $\iota_i:=\mathbf1_{\{i\in\mathcal J_{\mu_k}(N)\}}$. For each $i<M$, the option-level Bellman recursion~\citep[proof of Theorem~2]{sutton1999between} gives
\begin{equation}
 \begin{aligned}
 &\mathbb E_{\mu_k}\!\left[q(x_{i+1,0})-q(x_{i,0})+V^{\mu_\infty}(x_{i+1,0})\mid x_{i,0}\right]-V^{\mu_\infty}(x_{i,0})\\
 &\qquad=\mathbb E_{\mu_k}\!\left[\iota_i A^{\mu_\infty}(x_{i,-1})\mid x_{i,0}\right].
 \end{aligned}
 \label{eq:app-option-bellman-difference}
\end{equation}
Since $V^{\mu_\infty}(x_{M,0})=0$ and $x_{0,0}=x_0\sim\nu_0$, summing telescopes:
\begin{equation}
 \begin{aligned}
 \operatorname{Perf}_{\mu_k}(N)-\operatorname{Perf}_{\mu_\infty}(N)
 &=\mathbb E_{\mu_k}\!\left[q(x_{M,0})-q(x_{0,0})-V^{\mu_\infty}(x_{0,0})\right]\\
 &=\mathbb E_{\mu_k}\sum_{i\in\mathcal J_{\mu_k}(N)}A^{\mu_\infty}(x_{i,-1})
 =\mathcal A_{\mu_k}(N).
 \end{aligned}
 \label{eq:app-interruption-identity}
\end{equation}
Thus $\operatorname{Perf}_{\mu_k}(N)-U_*=\Delta_{\mu_\infty}(N)+\mathcal A_{\mu_k}(N)$, which proves~\eqref{eq:app-interruption-condition}.
\end{proof}

\begin{corollary}[A sufficient cap in a finite candidate range]
\leavevmode\phantomsection\label{cor:app-candidate-cap}
Fix $A,B,u_0,N$ with $\Delta_{\mu_\infty}(N)>0$ and integer $K\ge1$. For each candidate cap $1\le k\le K$, let $\mu_k=(A,B,u_0,k)$ and suppose there is $\varepsilon_{\mu_k}\ge0$ with $\mathcal A_{\mu_k}(N)\ge-\varepsilon_{\mu_k}\mathbb E_{\mu_k}|\mathcal J_{\mu_k}(N)|$. Then
\begin{equation}
 \min_{1\le k\le K}\lfloor N/k\rfloor\varepsilon_{\mu_k}<\Delta_{\mu_\infty}(N)
 \quad\Longrightarrow\quad
 \exists\,1\le k\le K:\operatorname{Perf}_{\mu_k}(N)>U_*.
 \label{eq:hma-small-cap-existence}
\end{equation}
\end{corollary}

\begin{proof}[Proof of Corollary~\ref{cor:app-candidate-cap}]\itshape
Each added interruption uses a distinct block of $k$ experiments, so $|\mathcal J_{\mu_k}(N)|\le\lfloor N/k\rfloor$ and $\mathcal A_{\mu_k}(N)\ge-\lfloor N/k\rfloor\varepsilon_{\mu_k}$. Choose a minimizing cap in~\eqref{eq:hma-small-cap-existence} and apply Theorem~\ref{thm:app-interruption}.
\end{proof}

\begin{corollary}[Explicit cap under geometric loss decay; quantitative form of Corollary~\ref{cor:hma-sufficient-cap}]
\leavevmode\phantomsection\label{cor:app-explicit-cap}
Fix $A,B,u_0,N$ with $\Delta_{\mu_\infty}(N)>0$. Under the expected loss condition of Corollary~\ref{cor:app-candidate-cap}, assume $\varepsilon_{\mu_k}\le C\rho^k$ for every $\mu_k=(A,B,u_0,k)$, $k\ge1$, with constants $C>0$ and $0<\rho<1$. The policy $\mu_{k_0}:=(A,B,u_0,k_0)$ with the following cap preserves at least half the \nta{} margin:
\begin{equation}
 \begin{aligned}
 k_0&:=\max\!\left\{1,\left\lceil\frac{\log(2NC/\Delta_{\mu_\infty}(N))}{-\log\rho}\right\rceil\right\},\\
 \operatorname{Perf}_{\mu_{k_0}}(N)&\ge U_*+\tfrac12\Delta_{\mu_\infty}(N)>U_*.
 \end{aligned}
 \label{eq:hma-explicit-cap}
\end{equation}
\end{corollary}

\begin{proof}[Proof of Corollary~\ref{cor:app-explicit-cap}]\itshape
The expected loss is at most $\lfloor N/k\rfloor C\rho^k\le NC\rho^k$. The choice of $k_0$ makes this at most $\Delta_{\mu_\infty}(N)/2$; substitution in~\eqref{eq:app-interruption-identity} proves the claim. Taking $k=k_0$ for the selected initial agent gives a sufficient cap in $[1,k_0]$, proving Corollary~\ref{cor:hma-sufficient-cap}. Uniform $C,\rho$ and a positive margin lower bound across budgets give $k_0=O(\log N)$.
\end{proof}

\subsection{Prefix Retention and Empirical Cap Selection}
\label{app:empirical}

We define the task-level prefix-retention statistic and report the empirical prefix curves supporting Section~\ref{sec:hma-fixed-k} (Figure~\ref{fig:hma-prefix-retention}).

\begin{definition}[Observed prefix gain]
\leavevmode\phantomsection\label{def:empirical-prefix}
For policy $\mu$ and task $j$, define
\[
 \mathcal O_{\mu,j}:=\{i:o_i^{(u_i)}\text{ is recorded on task }j\text{ under }\mu\}.
\]
For each session index $i\in\mathcal O_{\mu,j}$, let $l_i$ be its recorded experiment count and $q_i(d)$ its best quality after $d$ experiments, with entry quality at $d=0$. For integer $k\ge0$, define
\begin{equation}
 g_i(k)=q_i(\min\{k,l_i\})-q_i(0),
 \qquad \zeta_{\mu,j}(k)=
 \frac{\sum_{i\in\mathcal O_{\mu,j}}g_i(k)}{\sum_{i\in\mathcal O_{\mu,j}}g_i(l_i)}.
 \label{eq:app-prefix-retention}
\end{equation}
The ratio is defined for tasks with positive total gain.
\end{definition}

\paragraph{Observed prefix retention.}
We evaluate eight MLE-bench configurations across all 75 tasks with \texttt{max} reasoning effort: the five single-agent \texttt{/goal} configurations in Table~\ref{tab:hma-agent-configurations} and three \nta{} pairs (GPT-5.6-sol $\to$ Kimi K3, Kimi K3 $\to$ GPT-5.6-sol, and Claude Opus 5 $\to$ GPT-5.6-sol), yielding 746 nonempty options under normalized leaderboard quality $q=1-(\mathrm{position}-1)/(\mathrm{leaderboard\ size})$. At $k=1,3,5$, Figure~\ref{fig:hma-prefix-retention} records mean gain retentions of 65.84\%, 93.59\%, and 97.62\% at within-option experiment fractions of 23.34\%, 52.19\%, and 70.43\%, respectively, with individual configurations retaining 96.11\%--99.20\% at $k=5$. At $k=5$, 87.84\%--98.67\% of tasks retain at least 90\% of their full-session gains across the eight configurations.

\clearpage
\section{Additional Evidence for the Motivation}
\label{app:motivation-evidence}

This section provides the detailed measurements for Section~\ref{sec:hma-motivation}: Appendix~\ref{app:goal-token-growth} defines the trailing response-token statistic and compares early versus late single-agent responses; Appendix~\ref{app:harness-staleness-results} reports task-level outcomes for the DeepSeek harness comparison; Appendix~\ref{app:gpt-harness-transfer} presents the supplementary harness comparison across four GPT backbones.

\subsection{Declining Response Length in Single-Agent Runs}
\label{app:goal-token-growth}

\paragraph{Configurations.}
We document Figure~\ref{fig:hma-token-growth} using the five configurations in Table~\ref{tab:hma-agent-configurations}. The \texttt{/goal} protocol runs one agent until natural termination or exhaustion of six hours of active runtime, without model alternation or fixed-cap interruption. Each of the five configurations is evaluated on 75 tasks with three independent repeats and \texttt{max} reasoning effort.

\paragraph{Response-level moving average.}
Figure~\ref{fig:hma-token-growth} reports mean output tokens per recorded model response. At one-minute active-runtime checkpoints, we divide the sum of provider-reported output tokens by the number of responses in the preceding 30 minutes, pooling the 75 tasks and repeats within each configuration. This response-weighted statistic excludes tool-return text, input tokens, and coordinator output, while retaining model responses that request tools.

\paragraph{Early versus late responses.}
Table~\ref{tab:hma-response-token-change} compares responses in hours 0--2 and 3--5 using raw event totals divided by response counts in each interval. All five configurations show a reduction of 29.8\%--45.0\%; these interval statistics are computed before temporal smoothing.

\begin{table}[!htbp]
\centering
\footnotesize
\caption{\textbf{Mean output tokens per recorded response.} Change is relative to hours 0--2. The same task suite and repeats are used as in Figure~\ref{fig:hma-token-growth}.}
\label{tab:hma-response-token-change}
\setlength{\tabcolsep}{0pt}
\renewcommand{\arraystretch}{1.1}
\begin{tabular*}{0.8\linewidth}{@{\extracolsep{\fill}}l r r r@{}}
\toprule
Model & \makecell{Hours 0--2\\tokens/response} & \makecell{Hours 3--5\\tokens/response} & Change (\%) \\
\midrule
GPT-5.6-sol & 336.5 & 213.6 & $-36.5$ \\
Claude Opus 5 & 1008.3 & 554.8 & $-45.0$ \\
DeepSeek-V4.1-Flash & 1137.5 & 696.1 & $-38.8$ \\
Kimi K3 & 585.7 & 390.2 & $-33.4$ \\
DeepSeek-V4-Flash-Preview & 682.2 & 479.2 & $-29.8$ \\
\bottomrule
\end{tabular*}
\end{table}

\subsection{Harness Staleness across Model Versions}
\label{app:harness-staleness-results}

\paragraph{Comparison setup.}
We compare ML-Master 2.0~\citep{zhu2026mlmaster2}, MLEvolve~\citep{du2026mlevolve}, and ScienceFlow~\citep{zhao2026scienceflow} with the official DeepSeek Harness on a 16-task MLE-bench subset~\citep{chan2024mlebench}, using DeepSeek-V4-Flash-Preview and DeepSeek-V4.1-Flash. We retain the research harnesses' original workflow designs and experimental settings while disabling task-specific cold-start knowledge wherever present. Execution budgets are reported in Table~\ref{tab:hma-harness-staleness}. DeepSeek Harness runs each model under \texttt{/goal} with \texttt{max} reasoning effort for six hours; its reported medal rate is averaged over three repeats.

\paragraph{Result accounting.}
Medal rates use all 16 tasks as the denominator, counting runs without a valid submission as no medal. Table~\ref{tab:hma-harness-staleness} reports medal rates and percentage-point differences from the corresponding DeepSeek Harness baseline, and Table~\ref{tab:hma-harness-per-task} provides the task-level outcomes. These 16-task DeepSeek Harness references are distinct from the full 75-task results in Table~\ref{tab:hma-main-results}.

\begin{table}[!htbp]
\centering
\caption{\textbf{Task-level outcomes for the harness comparison.} DeepSeek Harness entries report the number of runs earning a medal out of three; its medal total is the mean count. The last row divides each medal total by 16, matching Table~\ref{tab:hma-harness-staleness}. Other columns report medal classes: G, S, and B denote gold, silver, and bronze, and \textemdash\ denotes no medal. V4 and V4.1 denote DeepSeek-V4-Flash-Preview and DeepSeek-V4.1-Flash.}
\label{tab:hma-harness-per-task}
\setlength{\tabcolsep}{4pt}
\renewcommand{\arraystretch}{1.05}
\resizebox{0.8\linewidth}{!}{%
\begin{tabular}{@{}l|cc|cc|cc|cc@{}}
\toprule
\textbf{Workflow}
& \multicolumn{2}{c|}{\textbf{DeepSeek Harness}}
& \multicolumn{2}{c|}{\textbf{ScienceFlow}}
& \multicolumn{2}{c|}{\textbf{ML-Master 2.0}}
& \multicolumn{2}{c}{\textbf{MLEvolve}} \\
\cmidrule(lr){2-3}\cmidrule(lr){4-5}\cmidrule(lr){6-7}\cmidrule(lr){8-9}
\textbf{Task} & V4 & V4.1 & V4 & V4.1 & V4 & V4.1 & V4 & V4.1 \\
\midrule
3D Object Detection & 1/3 & 2/3 & S & \textbf{G} & \textemdash & S & \textemdash & \textemdash \\
APTOS Blindness Detection & 0/3 & 3/3 & S & B & \textemdash & \textbf{G} & \textemdash & S \\
Toxic Comment Classification & 0/3 & 3/3 & \textemdash & \textbf{G} & \textemdash & \textemdash & B & S \\
Leaf Classification & 1/3 & 3/3 & S & S & \textemdash & \textemdash & B & \textbf{G} \\
Billion Word Imputation & 0/3 & 3/3 & \textemdash & \textbf{G} & \textemdash & \textemdash & \textemdash & \textemdash \\
Cdiscount Image Classification & 0/3 & 1/3 & \textemdash & \textemdash & \textemdash & \textemdash & \textemdash & \textemdash \\
Chaii Question Answering & 0/3 & 1/3 & \textemdash & \textemdash & \textemdash & \textemdash & \textemdash & \textemdash \\
Freesound Audio Tagging & 1/3 & 1/3 & \textemdash & \textemdash & \textemdash & \textemdash & \textemdash & S \\
iMet & 0/3 & 1/3 & \textemdash & S & \textemdash & \textemdash & \textemdash & \textemdash \\
iWildCam & 2/3 & 3/3 & \textemdash & \textbf{G} & \textemdash & S & \textbf{G} & S \\
Kuzushiji Recognition & 2/3 & 3/3 & B & \textemdash & \textemdash & \textbf{G} & \textemdash & \textemdash \\
Learning Agency Essay Scoring & 0/3 & 3/3 & \textemdash & \textbf{G} & \textemdash & \textbf{G} & \textemdash & \textemdash \\
LMSYS Chatbot Arena & 0/3 & 2/3 & \textemdash & S & \textemdash & \textemdash & \textemdash & \textemdash \\
PetFinder Pawpularity & 0/3 & 3/3 & \textemdash & \textbf{G} & \textemdash & S & \textemdash & S \\
Tweet Sentiment Extraction & 0/3 & 2/3 & \textemdash & S & \textemdash & \textemdash & \textemdash & S \\
US Patent Phrase Matching & 1/3 & 3/3 & \textemdash & \textemdash & S & \textemdash & \textemdash & S \\
\midrule
\rowcolor{gray!12}
\textbf{Medals} & \textbf{2.7} & \textbf{12.3} & \textbf{4} & \textbf{11} & \textbf{1} & \textbf{6} & \textbf{3} & \textbf{8} \\
\textbf{Medal rate (\%)} & 16.7 & 77.1 & 25.0 & 68.8 & 6.3 & 37.5 & 18.8 & 50.0 \\
\bottomrule
\end{tabular}%
}
\end{table}

\subsection{Supplementary Harness Comparison across GPT Backbones}
\label{app:gpt-harness-transfer}

\paragraph{Experiment configuration.}
We use the task subset and research harnesses from Appendix~\ref{app:harness-staleness-results} with native Codex as the reference to evaluate model--harness native alignment across four GPT backbones. Each configuration runs once per task for six hours. GPT-5.5 uses \texttt{xhigh} and the other three backbones use \texttt{max}. Table~\ref{tab:hma-gpt-harness-transfer} evaluates the latest accepted submission on each task and counts missing submissions as no medal. Across all four GPT backbones, all 12 third-party model--harness configurations trail native Codex by 6.3--68.8 percentage points, demonstrating the advantage of native model--harness alignment.

\begin{table}[!htbp]
\centering
\footnotesize
\caption{\textbf{Specialized harnesses trail native Codex across four GPT backbones.} Medal rates (\%) on the 16-task MLE-bench subset. Native Codex is the same-model baseline; colored differences are percentage points relative to it. Bold and underlined entries indicate the largest and second-largest rates within each model column.}
\label{tab:hma-gpt-harness-transfer}
\setlength{\tabcolsep}{0pt}
\renewcommand{\arraystretch}{1.1}
\begin{tabular*}{0.8\linewidth}{@{\extracolsep{\fill}}l r r r r@{}}
\toprule
\textbf{Harness}
& \makecell{\textbf{GPT-5.5}\\\textbf{Medals (\%)}}
& \makecell{\textbf{GPT-5.6-sol}\\\textbf{Medals (\%)}}
& \makecell{\textbf{GPT-6-sol}\\\textbf{Medals (\%)}}
& \makecell{\textbf{GPT-6.1-sol}\\\textbf{Medals (\%)}} \\
\midrule
Native Codex & \textbf{50.0} & \textbf{93.8} & \textbf{87.5} & \textbf{100.0} \\
ScienceFlow & 37.5\,{\scriptsize\textcolor{red!70!black}{$-12.5$}} & \underline{68.8}\,{\scriptsize\textcolor{red!70!black}{$-25.0$}} & 50.0\,{\scriptsize\textcolor{red!70!black}{$-37.5$}} & \underline{68.8}\,{\scriptsize\textcolor{red!70!black}{$-31.3$}} \\
ML-Master 2.0 & 12.5\,{\scriptsize\textcolor{red!70!black}{$-37.5$}} & 43.8\,{\scriptsize\textcolor{red!70!black}{$-50.0$}} & 56.3\,{\scriptsize\textcolor{red!70!black}{$-31.3$}} & 37.5\,{\scriptsize\textcolor{red!70!black}{$-62.5$}} \\
MLEvolve & \underline{43.8}\,{\scriptsize\textcolor{red!70!black}{$-6.3$}} & 25.0\,{\scriptsize\textcolor{red!70!black}{$-68.8$}} & \underline{62.5}\,{\scriptsize\textcolor{red!70!black}{$-25.0$}} & 50.0\,{\scriptsize\textcolor{red!70!black}{$-50.0$}} \\
\bottomrule
\end{tabular*}
\end{table}

\clearpage
\section{Experimental Setup and Reproducibility}
\label{app:experiment-details}

This section documents the experimental setup in Section~\ref{sec:hma-experiment-setup}, including the execution environment, model--harness mappings and comparison baselines, and the evaluation and repeat protocol.

\vspace{-2pt}
\begin{table}[h]
\centering
\footnotesize
\caption{Model--harness mappings and reasoning effort for our runs.}
\label{tab:hma-agent-configurations}
\setlength{\tabcolsep}{0pt}
\renewcommand{\arraystretch}{1.1}
\begin{tabular*}{0.8\linewidth}{@{\extracolsep{\fill}}lll@{}}
\toprule
\textbf{Model} & \textbf{Harness} & \textbf{Reasoning effort} \\
\midrule
GPT-5.6-sol & Codex & \texttt{max} \\
Kimi K3 & Kimi Code & \texttt{max} \\
DeepSeek-V4-Flash-Preview & DeepSeek Harness & \texttt{max} \\
DeepSeek-V4.1-Flash & DeepSeek Harness & \texttt{max} \\
Claude Opus 5 & Claude Code & \texttt{max} \\
\bottomrule
\end{tabular*}
\end{table}

\paragraph{Execution environment and agent configurations.}
Each six-hour task run allocates 20,700 seconds to exploration and 900 seconds to final review on a machine with 30 vCPUs (Intel Xeon Platinum 8358), 220 GiB RAM, and one 24 GiB NVIDIA A10 GPU shared by the alternating agents. Table~\ref{tab:hma-agent-configurations} lists the models, native harnesses~\citep{deepseek-harness2026,bolin2026codexloop,anthropicClaudeCodeWorks,moonshot2026kimicode}, and \texttt{max} reasoning effort used in our runs; within-family controls additionally pair Claude Opus 5 with Claude Sonnet 5 and GPT-5.6-sol with GPT-5.6-terra under the same protocol. In prose and figure labels, GPT, Kimi, Opus, DeepSeek-V4, and DeepSeek-V4.1 occasionally abbreviate these full model names. Table~\ref{tab:hma-baseline-configurations} summarizes the original configurations of the published comparison baselines~\citep{zhao2026scienceflow,du2026mlevolve,zhu2026mlmaster2}.

\vspace{-2pt}
\begin{table}[h]
\centering
\footnotesize
\caption{Configurations of the three comparison baselines, as reported in their original papers. Resource allocations and time budgets differ across methods.}
\label{tab:hma-baseline-configurations}
\setlength{\tabcolsep}{0pt}
\renewcommand{\arraystretch}{1.1}
\begin{tabular*}{\linewidth}{@{\extracolsep{\fill}}>{\raggedright\arraybackslash}p{.11\linewidth}>{\raggedright\arraybackslash}p{.25\linewidth}>{\raggedright\arraybackslash}p{.27\linewidth}>{\raggedright\arraybackslash}p{.30\linewidth}@{}}
\toprule
\textbf{Baseline} & \textbf{Model and time budget} & \textbf{Execution resources} & \textbf{Key configuration} \\
\midrule
ScienceFlow & DeepSeek-V4-Flash-Preview; 24 h & 16--32 logical CPU cores, 256 GB RAM, at most 2 GPUs & Recoverable workspaces, ESTRA re-anchoring, and resource-aware execution control. \\
MLEvolve & Gemini-3.1-Pro-preview; 12 h & 21 vCPUs, 234 GB RAM, 1 NVIDIA H200 & Temperature 1.0; 500 expansion steps; 3 parallel branches and 3 initial drafts. \\
ML-Master 2.0 & DeepSeek-V3.2-Speciale; DeepSeek-V3.2 (thinking) for context promotion; 24 h & 36 AMD EPYC vCPUs and 2 RTX 4090 GPUs per agent; 1008 GB RAM and 1 TB SSD shared across 4 tasks & Hierarchical cognitive caching; prior wisdom initialized from 407 Kaggle competitions excluding MLE-bench. \\
\bottomrule
\end{tabular*}
\end{table}

\vspace{-2pt}
\begin{table}[h]
\centering
\footnotesize
\caption{Per-run any-medal rates (\%), three-run means, and sample standard errors ($\mathrm{SE}$) for heterogeneous \method{} across the 75 MLE-bench tasks.}
\label{tab:hma-repeat-rates}
\setlength{\tabcolsep}{0pt}
\renewcommand{\arraystretch}{1.1}
\begin{tabular*}{0.8\linewidth}{@{\extracolsep{\fill}}lrrrrr@{}}
\toprule
\textbf{HMA configuration} & \textbf{Run 1} & \textbf{Run 2} & \textbf{Run 3} & \textbf{Mean} & \textbf{SE} \\
\midrule
Claude Opus 5 $\to$ GPT-5.6-sol & 77.3\% & 78.7\% & 78.7\% & 78.2\% & 0.4\% \\
GPT-5.6-sol $\to$ Claude Opus 5 & 73.3\% & 78.7\% & 78.7\% & 76.9\% & 1.8\% \\
GPT-5.6-sol $\to$ DeepSeek-V4.1-Flash & 72.0\% & 73.3\% & 72.0\% & 72.4\% & 0.4\% \\
GPT-5.6-sol $\to$ Kimi K3 & 70.7\% & 70.7\% & 68.0\% & 69.8\% & 0.9\% \\
DeepSeek-V4.1-Flash $\to$ Kimi K3 & 62.7\% & 60.0\% & 61.3\% & 61.3\% & 0.8\% \\
\bottomrule
\end{tabular*}
\end{table}

\paragraph{Evaluation protocol and repeats.}
Task-suite metrics are computed per repeat using normalized leaderboard quality $q=1-(\mathrm{position}-1)/(\mathrm{leaderboard\ size})$ and then averaged across repeats. For three-run configurations (all \texttt{/goal} baselines and heterogeneous \method{} pairs), Table~\ref{tab:hma-main-results} reports the across-run mean and sample standard error ($\mathrm{SE}$); \textit{Avg.\ of Single} averages the two constituent means with $\mathrm{SE}=\sqrt{\mathrm{SE}_A^2+\mathrm{SE}_B^2}/2$, and signed gains are computed before rounding. Table~\ref{tab:hma-repeat-rates} reports the per-run any-medal rates for heterogeneous \method{}.

\clearpage
\section{Additional Analysis and Ablation Studies}
\label{app:ablations}

This section provides the detailed trajectory analyses and ablation studies supporting Section~\ref{sec:hma-ablation}: Appendix~\ref{app:main-curves} documents the time-resolved response-length and any-medal curves in Figure~\ref{fig:hma-all-growth}; Appendix~\ref{app:starting-order} analyzes starting-order sensitivity, partner participation timing, and session caps $k\in\{1,3,5\}$; Appendix~\ref{app:option-termination} examines option-closure reasons across sessions; Appendix~\ref{app:review} evaluates historical medal retention before and after final review.

\subsection{Performance and Response Trajectories}
\label{app:main-curves}

Figure~\ref{fig:hma-all-growth} pairs each of the five heterogeneous \method{} configurations with the mean of its two constituent \texttt{/goal} baselines and averages equally ($1/5$ each) across configurations (three runs per configuration), so the matched \texttt{/goal} curve weights GPT-5.6-sol by $2/5$ and Claude Opus 5, DeepSeek-V4.1-Flash, and Kimi K3 by $1/5$ each. From hours 0--2 to hours 3--5, mean output tokens per response change from 628.5 to 548.6 under \method{} (87.3\% retained) versus 680.9 to 413.7 under the matched \texttt{/goal} baseline (60.8\% retained); Table~\ref{tab:hma-early-medals} lists the corresponding any-medal rates at two and six hours.

\begin{table}[h]
\centering
\footnotesize
\caption{Any-medal rates (\%) in Figure~\ref{fig:hma-all-growth}(b). Six-hour rates equal the three-run means in Table~\ref{tab:hma-main-results}.}
\label{tab:hma-early-medals}
\setlength{\tabcolsep}{0pt}
\renewcommand{\arraystretch}{1.1}
\begin{tabular*}{0.8\linewidth}{@{\extracolsep{\fill}}lrr@{}}
\toprule
Starting order & 2 h & 6 h \\
\midrule
Claude Opus 5 $\to$ GPT-5.6-sol & 57.3 & 78.2 \\
GPT-5.6-sol $\to$ Claude Opus 5 & 60.9 & 76.9 \\
GPT-5.6-sol $\to$ DeepSeek-V4.1-Flash & 58.2 & 72.4 \\
GPT-5.6-sol $\to$ Kimi K3 & 57.3 & 69.8 \\
DeepSeek-V4.1-Flash $\to$ Kimi K3 & 44.0 & 61.3 \\
\midrule
Mean & 55.6 & 71.7 \\
\bottomrule
\end{tabular*}
\end{table}

\subsection{Starting Order and Interruption Cap}
\label{app:starting-order}

\begin{figure}[h]
\centering
\includegraphics[width=.8\linewidth]{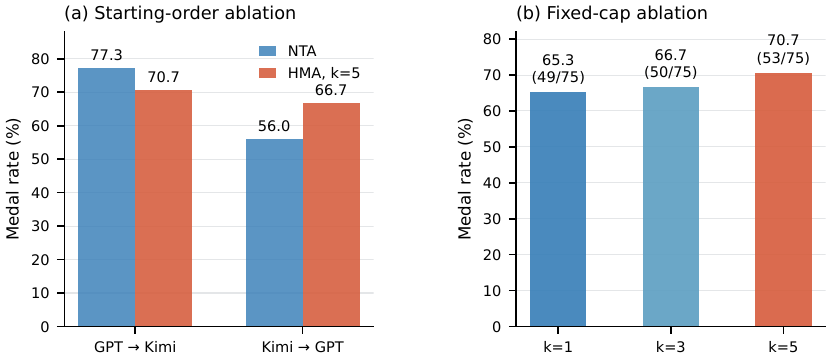}
\caption{\textbf{Fixed-cap interruption: starting-order and cap ablations.} Both panels use the same 75-task, six-hour GPT-5.6-sol--Kimi K3 suite. (a) Starting-order comparison: \nta{} and \method{} with $k=5$; the cap narrows the observed order gap. (b) GPT-first \method{} with $k\in\{1,3,5\}$, producing 49/75, 50/75, and 53/75 medals at $k=1$, $k=3$, and $k=5$, respectively. Values are final medal rates, with counts shown in the cap panel.}
\label{fig:hma-main-trajectories}
\end{figure}

\paragraph{Starting-order sensitivity.}
Figure~\ref{fig:hma-main-trajectories}(a) compares both GPT--Kimi starting orders across the 75 tasks under the six-hour budget. Under \nta{}, final any-medal rates are 77.3\% with GPT-5.6-sol first and 56.0\% with Kimi K3 first; \method{} ($k=5$) reaches 70.7\% and 66.7\% in the paired single-run ablation (69.8\% three-run mean for GPT-5.6-sol $\to$ Kimi K3), narrowing the starting-order gap from 21.3 to 4.0 percentage points and raising the two-order mean from 66.7\% to 68.7\%. On the third \nta{} pair in Appendix~\ref{app:empirical} (Claude Opus 5 $\to$ GPT-5.6-sol), where both constituents are comparably strong, \method{} ($k=5$) also improves over \nta{} from 56/75 (74.7\%) to 58/75 (77.3\% in Run 1; 78.2\% three-run mean).

\paragraph{Active-runtime allocation.}
Measuring each model's active runtime $w_j^{(u)}$ on task $j$ via the mean budget share $100\sum_{j=1}^{75}w_j^{(u)}/(75\cdot21600)$, we find that the starting model under \nta{} occupies 81--83\% of the six-hour budget: when GPT-5.6-sol starts, Kimi K3 acts mainly as a terminal reviewer, whereas starting with the weaker Kimi K3 leaves GPT-5.6-sol too little time to recover. By contrast, capping sessions at $k=5$ balances active runtime in GPT-first \method{} (51.1\% for GPT-5.6-sol and 46.3\% for Kimi K3, with 2.6\% transition overhead) and prevents Kimi K3 from exhausting the six-hour budget in Kimi-first \method{}, lifting the Kimi-first medal rate by 10.7 percentage points.

\paragraph{Partner participation timing.}
Table~\ref{tab:hma-handoff-timing} tracks when the second agent first produces an exploration response (excluding final review). Under \nta{}, the second agent enters exploration by hour 5.75 on only 19/75 GPT-first tasks and 47/75 Kimi-first tasks. Capping sessions at $k=5$ under \method{} advances partner entry in both directions: in GPT-first \method{}, Kimi K3 enters on all 75/75 tasks, raising its mean within-task share of output tokens before hour 5.75 from 13.0\% to 36.1\%; in Kimi-first \method{}, GPT-5.6-sol enters on 73/75 tasks by hour 5.75, lifting the final medal rate from 56.0\% to 66.7\%.

\begin{table}[h]
\centering
\caption{\textbf{The cap advances access to the second agent.} All counts are out of 75 tasks and exclude final review.}
\label{tab:hma-handoff-timing}
\setlength{\tabcolsep}{5pt}
\renewcommand{\arraystretch}{1.1}
\resizebox{0.8\linewidth}{!}{%
\begin{tabular}{@{}l r r r r r@{}}
\toprule
& \multicolumn{3}{c}{Tasks with a second-agent response}
& \makecell{Median first\\response (h)} & \makecell{Final medals\\(\%)} \\
\cmidrule(lr){2-4}
Workflow & By 3 h & By 5 h & By 5.75 h & & \\
\midrule
GPT-5.6-sol $\to$ Kimi K3 \nta{} & 17 & 18 & 19 & 5.97 & 77.3 \\
Kimi K3 $\to$ GPT-5.6-sol \nta{} & 15 & 27 & 47 & 5.56 & 56.0 \\
GPT-5.6-sol $\to$ Kimi K3 \method{} & 55 & 69 & 75 & 1.89 & 70.7 \\
Kimi K3 $\to$ GPT-5.6-sol \method{} & 15 & 40 & 73 & 4.91 & 66.7 \\
\bottomrule
\end{tabular}%
}
\end{table}

\paragraph{Session-cap sensitivity ($k\in\{1,3,5\}$).}
Motivated by the empirical prefix-retention curves in Appendix~\ref{app:empirical}, we evaluate $k\in\{1,3,5\}$ under the same six-hour GPT-5.6-sol $\to$ Kimi K3 protocol (Figure~\ref{fig:hma-main-trajectories}b). Caps of $k=1,3,5$ yield 49, 50, and 53 medals (65.3\%, 66.7\%, and 70.7\%), confirming that $k=5$ best balances within-session progress and handoff frequency.

\subsection{Natural Termination before the Cap}
\label{app:option-termination}

\begin{table}[h]
\centering
\footnotesize
\caption{\textbf{Observed option-closure reasons at $k=5$.} Each configuration covers three runs of 75 tasks (225 initial options). Natural-exit percentages exclude deadline closures because the latter censor whether an option would reach the cap or terminate naturally. The first listed model starts each workflow.}
\label{tab:hma-option-closures}
\setlength{\tabcolsep}{0pt}
\renewcommand{\arraystretch}{1.1}
\begin{tabular*}{0.8\linewidth}{@{\extracolsep{\fill}}l r r r r r r@{}}
\toprule
& \multicolumn{3}{c}{Initial option} & \multicolumn{3}{c}{Later options} \\
\cmidrule(lr){2-4}\cmidrule(lr){5-7}
Configuration & Cap & Natural & Deadline & Cap & Natural & Deadline \\
\midrule
Claude Opus 5 $\to$ GPT-5.6-sol & 183 & 41 & 1 & 244 & 451 & 224 \\
GPT-5.6-sol $\to$ Claude Opus 5 & 210 & 13 & 2 & 243 & 225 & 223 \\
GPT-5.6-sol $\to$ DeepSeek-V4.1-Flash & 205 & 20 & 0 & 281 & 458 & 224 \\
GPT-5.6-sol $\to$ Kimi K3 & 209 & 14 & 2 & 164 & 315 & 223 \\
DeepSeek-V4.1-Flash $\to$ Kimi K3 & 176 & 46 & 3 & 53 & 430 & 222 \\
\midrule
\textbf{All configurations} & 983 & 134 & 8 & 985 & 1,879 & 1,116 \\
\bottomrule
\end{tabular*}
\end{table}

Table~\ref{tab:hma-option-closures} reports option-closure outcomes across all three runs of each of the five heterogeneous configurations (1,125 task runs), separating the initial option from later options across cap interruptions ($k=5$), natural termination, and exploration-deadline closures. Among options closing before the deadline, 983/1,117 initial options reach the cap (88.0\%), whereas 1,879/2,864 later options terminate naturally (65.6\%; 2,013/3,981 or 50.6\% overall). This pattern shows that \method{}'s hybrid termination rule $d_i=\min\{T_i^{(u_i)}(x_i),k,t_i\}$ does not mechanically force five experiments in every session: in the opening session, $k=5$ acts as a safeguard that prevents the starting agent from monopolizing the budget, whereas in subsequent sessions---once a working pipeline and experiment history are in place---the incoming agent typically completes a few targeted refinements and hands off naturally.

\subsection{Final Review and Historical Any-Medal Retention}
\label{app:review}

Across the 1,125 task runs in Appendix~\ref{app:option-termination}, 830 runs archived at least one bronze-or-better candidate during exploration. Comparing the standing candidate immediately before final review with the returned candidate afterward, final review increased selected medals from 789 to 807 (a net gain of 18 medals), raising historical any-medal retention from 95.1\% to 97.2\% (Table~\ref{tab:hma-review-ablation}; conditional selection bound in Appendix~\ref{app:returned-guarantee}). For example, in one GPT-5.6-sol $\to$ DeepSeek-V4.1-Flash run, final review recovered a bronze-medal submission (0.85445) archived by DeepSeek-V4.1-Flash after a subsequent GPT-5.6-sol session regressed to 0.75795.

\begin{table}[!htbp]
\centering
\caption{\textbf{Historical any-medal retention.} Each row covers three runs of 75 tasks (225 runs). Retention divides selected medal counts by runs with an archived medal. Review changed 88 selections across the five configurations.}
\label{tab:hma-review-ablation}
\setlength{\tabcolsep}{4pt}
\renewcommand{\arraystretch}{1.1}
\resizebox{0.8\linewidth}{!}{%
\begin{tabular}{@{}l r r c c r@{}}
\toprule
Configuration & Changed & \makecell{Historical\\any medal} & \makecell{Selected medals\\before $\to$ after} & \makecell{Retention (\%)\\before $\to$ after} & \makecell{Missed\\after} \\
\midrule
Claude Opus 5 $\to$ GPT-5.6-sol & 24 & 177 & $170 \to 176$ & $96.0 \to 99.4$ & 1 \\
GPT-5.6-sol $\to$ Claude Opus 5 & 16 & 177 & $171 \to 173$ & $96.6 \to 97.7$ & 4 \\
GPT-5.6-sol $\to$ DeepSeek-V4.1-Flash & 29 & 170 & $155 \to 163$ & $91.2 \to 95.9$ & 7 \\
GPT-5.6-sol $\to$ Kimi K3 & 11 & 160 & $155 \to 157$ & $96.9 \to 98.1$ & 3 \\
DeepSeek-V4.1-Flash $\to$ Kimi K3 & 8 & 146 & $138 \to 138$ & $94.5 \to 94.5$ & 8 \\
\midrule
\textbf{All configurations} & 88 & 830 & $789 \to 807$ & $95.1 \to 97.2$ & 23 \\
\bottomrule
\end{tabular}%
}
\end{table}
\clearpage
\section{Case Studies of Heterogeneous Alternation}
\label{app:case-studies}

This section presents three task-level case studies spanning four model families: Appendix~\ref{app:case-ingv} examines repeated cross-agent feature and constraint refinement on INGV Volcanic Eruption Prediction (Section~\ref{sec:hma-case-study}); Appendix~\ref{app:case-quest} analyzes prediction reuse and blending on Google QUEST Q\&A Labeling; Appendix~\ref{app:case-gnss-dskimi} traces late trajectory repair on the Google Smartphone Decimeter Challenge.

\paragraph{Measurement protocol.}
Following the trajectory analysis of \citet{zhao2026scienceflow}, we track the offline test score of every accepted candidate across sessions, including regressions; test scores remain hidden from the agents during execution, and \texttt{/goal} references use the best of three runs per constituent. For response-length dynamics, we compare the mean output tokens of the first 20 and last 20 responses within each exploration option (excluding final review, which preserves the standing candidate in all three cases).

\subsection{INGV Volcanic Eruption Prediction}
\label{app:case-ingv}

\paragraph{Repeated refinement.}
Five of six handoffs improved the best archived MAE on INGV Volcanic Eruption Prediction. GPT reused and pruned Opus's retrieval features, Opus added consistency constraints for paired windows, and later sessions refined residual gating, calibration, and joint temporal feasibility. MAE fell from 41,214 after the first Opus session to 34,234 at return, 46.5\% below the best of six constituent \texttt{/goal} runs (63,991). Native MAE distinguishes these improvements even after both workflows reach gold.

In Figure~\ref{fig:hma-case-ingv}(b), each point averages response length from the preceding acceptance or option start to the next acceptance. Comparing each model's first experiment on return with its preceding last experiment, mean response length increased on four of five returns (1.6--3.6 times).

\subsection{Google QUEST Q\&A Labeling: GPT--DeepSeek}
\label{app:case-quest}

\begin{figure}[h]
\centering
\includegraphics[width=\linewidth]{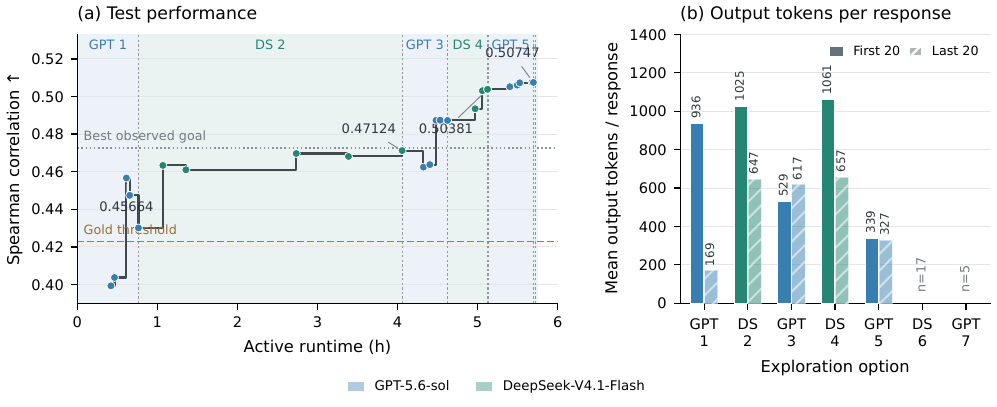}
\caption{\textbf{GPT--DeepSeek refinement on Google QUEST.} (a) All 23 accepted scores, including regressions. Shading identifies the active model; reference lines mark the best constituent \texttt{/goal} endpoint and gold threshold. (b) First- and last-20-response means per option. Options 6 and 7 terminate early with 17 and 5 responses and no accepted experiments. Final review retains the standing candidate and is excluded from token statistics.}
\label{fig:hma-case-quest}
\end{figure}

\paragraph{Repeated reuse of text-model predictions.}
On Google QUEST Q\&A Labeling, GPT initialized sparse-text and encoder models. DeepSeek fitted additional DeBERTa models and combined their saved predictions. Later sessions reused both models' predictions for target-specific blending and calibration. The pair reached a best Spearman correlation of 0.50760 and returned 0.50747, above the best constituent \texttt{/goal} result of 0.47252. The final two sessions produced no accepted submission (Figure~\ref{fig:hma-case-quest}).

\subsection{Google Smartphone Decimeter: DeepSeek--Kimi}
\label{app:case-gnss-dskimi}

\begin{figure}[h]
\centering
\includegraphics[width=\linewidth]{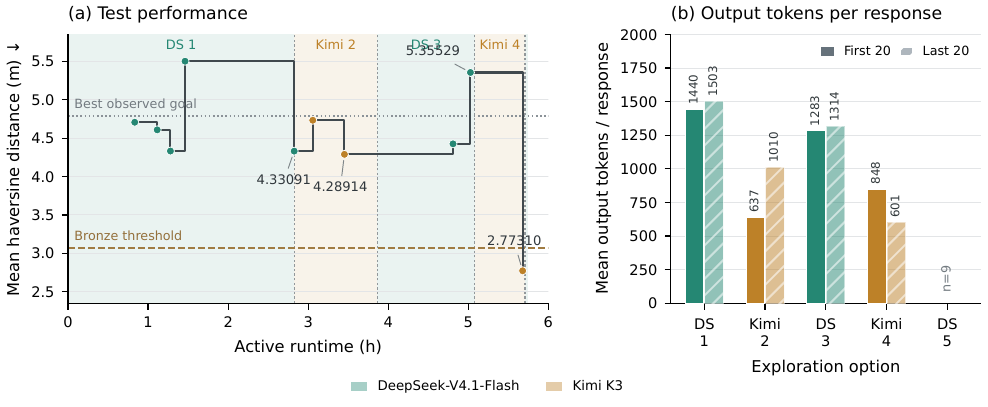}
\caption{\textbf{DeepSeek--Kimi correction in smartphone positioning.} (a) All ten accepted mean-haversine scores. The fourth option repairs the late outage and crosses the bronze threshold after the third option regresses. The dotted line is the best constituent \texttt{/goal} endpoint. (b) First- and last-20-response means; the final DeepSeek option contains nine responses and no accepted experiment. Final review retains Kimi's repaired candidate.}
\label{fig:hma-case-gnss-dskimi}
\end{figure}

\paragraph{Late correction after an unsuccessful refinement.}
On the Google Smartphone Decimeter Challenge, DeepSeek built a GNSS fusion pipeline and Kimi refined it. After a subsequent regression, Kimi returned at hour 5.08 and replaced the solver's straight-line extrapolation through a signal outage with a direction-matched route from public training trajectories. Mean haversine error fell to 2.77310, earning bronze and improving on the best constituent \texttt{/goal} endpoint (4.78536) by 42.1\% (Figure~\ref{fig:hma-case-gnss-dskimi}).

\end{document}